\documentclass[final,12pt]{clear2026} 

\title[Learning Chain-Reaction Mechanisms from Interventions]{Causal Discovery in Action: Learning Chain-Reaction Mechanisms from Interventions}
\usepackage{times}
\usepackage{graphicx}
\usepackage[utf8]{inputenc} 
\usepackage[T1]{fontenc}    
\usepackage{hyperref}       
\usepackage{url}            
\usepackage{booktabs}       
\usepackage{amsfonts}       
\usepackage{nicefrac}       
\usepackage{microtype}      
\usepackage{xcolor}         
\usepackage{amssymb,amsmath}
\usepackage{mathtools} 
\usepackage{cleveref}
\usepackage{float}
\usepackage{natbib}
\usepackage{listings}
\usepackage{tikz}
\usepackage{algorithm}
\usepackage{algorithmic}
\usepackage{caption}
\usepackage{colortbl}
\usepackage{bbm}
\usepackage{subcaption}
\usepackage{float} 

\usetikzlibrary{automata, fit}
\usetikzlibrary{arrows.meta,positioning,calc}
\usetikzlibrary{decorations.pathreplacing}
\usetikzlibrary{decorations.markings}

\newcommand{\doop}[1]{\textit{\textbf{do}}(#1)}
\newcommand{\Pa}{\mathrm{Pa}}
\newcommand{\Desc}{\mathrm{Desc}}
\newcommand{\Anc}{\mathrm{Anc}}

\clearauthor{%
 \Name{Panayiotis Panayiotou} \Email{pp2024@bath.ac.uk}\\
 \addr Department of Computer Science, University of Bath, UK
 \AND
 \Name{{\"{O}}zg\"ur \c{S}im\c{s}ek} \Email{o.simsek@bath.ac.uk}\\
 \addr Department of Computer Science, University of Bath, UK%
}

\begin{document}

\maketitle
\thispagestyle{empty} 

\begin{abstract}%
Causal discovery is challenging in general dynamical systems because, without strong structural assumptions, the underlying causal graph may not be identifiable even from interventional data.
However, many real-world systems exhibit directional, cascade-like structure, in which components activate sequentially and upstream failures suppress downstream effects. We study causal discovery in such chain-reaction systems and show that the causal structure is uniquely identifiable from blocking interventions that prevent individual components from activating. We propose a minimal estimator with finite-sample guarantees, achieving exponential error decay and logarithmic sample complexity. Experiments on synthetic models and diverse chain-reaction environments demonstrate reliable recovery from a few interventions, while observational heuristics fail in regimes with delayed or overlapping causal effects.
\end{abstract}

\begin{keywords}%
    causal discovery,
    interventions,
    chain-reaction systems,
    mechanistic causal models
\end{keywords}

\section{Introduction}

Physical systems are composed of interacting components whose effects propagate through structured mechanisms. While causal discovery is challenging in general dynamical systems \cite{peters2017elements, runge2019inferring}, many engineered and natural systems exhibit strongly directional, cascade-like structure: components activate sequentially, and blocking an upstream component reliably suppresses all downstream effects. 
Such \emph{chain-reaction systems} arise in a wide range of settings, including mechanical safety interlocks and emergency-stop systems \citep{leveson2016engineering}, biological signaling and gene-regulatory cascades \citep{kauffman1969metabolic,shmulevich2002probabilistic}, relay and logic circuits, and software or infrastructure dependency graphs.
In these systems, the absence of an upstream activation propagates monotonically downstream: disabling a safety switch prevents all subsequent actuators from engaging, knocking out an upstream protein suppresses downstream gene expression, cutting power to a relay halts all dependent components, and removing a prerequisite service prevents dependent services from executing.

In this work, we study causal discovery in this restricted but practically relevant regime, and show that its strong mechanistic asymmetry enables simple and identifiable recovery of the underlying causal graph from blocking interventions. Specifically, we study \emph{chain-reaction systems} inspired by Rube Goldberg machines \citep{goldberg2013art}. These systems are intentionally constructed so that objects play \emph{asymmetric functional roles} in a directed cascade: a ball strikes a domino, the domino presses a button, the button releases a gate, and so on. See Figure~\ref{fig:intuition} for an illustrative example. While the underlying physical forces are symmetric at the level of Newtonian dynamics, the system itself is purposefully directional.

\begin{figure}[t]
\centering
\includegraphics[trim=40 110 240 40,clip,width=0.30\linewidth]{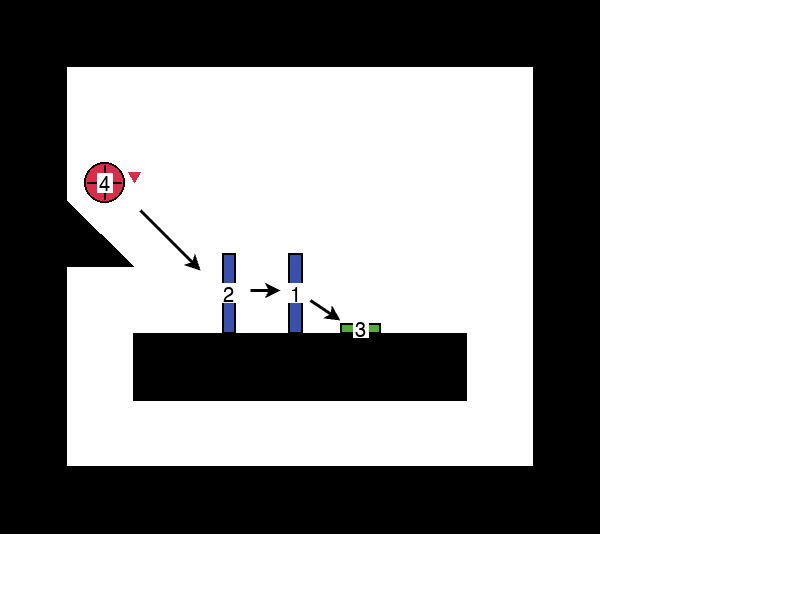}
\caption{
\textbf{Causal discovery in a chain-reaction system.}
Directed edges represent \emph{causal responsibility} rather than low-level physical forces.
Holding an object in place breaks the chain reaction and deterministically suppresses all downstream activations, revealing ancestor--descendant relations. 
}
\label{fig:intuition}
\end{figure}

A natural question is whether such causal structure can be inferred from observation alone. In a simple demonstration, the temporal order of object activations may reveal a clear causal chain, suggesting that tracking object motion or detecting collisions could suffice.
However, temporal order does not, in general, identify causation. 
Distinct causal graphs can induce identical activation time sequences when multiple objects activate simultaneously or when causal effects are delayed, for example when pressing a button releases a platform at a distance after some time.
Methods that track collisions have similar issues, since some interactions don't involve direct visible contact (such as a button press releasing a ball). More fundamentally, purely observational heuristics rely on strong assumptions about perception, such as accurately identifying collisions and attributing effects to the correct causes. Even the most sophisticated observational methods become brittle when multiple upstream objects provide equally plausible explanations for a downstream activation (e.g., two buttons pressed at the same time triggering different mechanisms). Such ambiguities cannot be resolved from observation alone and are precisely what causal interventions are designed to address \citep{eberhardt2006n, eberhardt2007interventions, pearl2009causality}.

We adopt a causal abstraction where each object is represented by a binary variable indicating whether it becomes \emph{active} (e.g., moves) during an episode, and a directed edge $i \!\to\! j$ indicates that object $i$ is causally responsible for triggering the activation of object $j$ as part of the chain reaction. The goal is to recover the directed chain-reaction graph from repeated interactions with the system.
We model interaction with the chain-reaction system as an interventional data collection process. In each execution (episode), the experimenter can apply a \emph{blocking intervention} that prevents a chosen object $i$ from activating while leaving the rest of the system intact, i.e., $\doop{X_i} = 0$. Intuitively, blocking an object ``breaks'' the chain reaction: downstream activations that depend on that object fail to activate, revealing which components are causally downstream.
Physical executions are imperfect, so we add small stochastic variations in the initial setup (slight displacements of objects when setting up the system), which can prevent the chain reaction from fully propagating, even in the absence of a blocking intervention (e.g., a domino may fall without triggering the next one). To model this stochasticity in a Structural Causal Model (SCM), we allow each variable to fail to activate with some probability even when all of its parents are active. This yields a simple SCM for chain reactions, enabling the recovery of the causal structure through blocking interventions.

We show that under this monotone cascade model, the causal structure of a chain-reaction mechanism is \emph{provably identifiable} from single-object blocking interventions (Theorem~\ref{thm:ident}). In particular, an object $j$ is a descendant (but not necessarily a child) of an intervened object $i$ if and only if $j$ is never observed to activate under $\doop{X_i = 0}$. Since we assume that each object has exactly one direct trigger, the causal graph is a directed tree, and it can be obtained via transitive reduction. We propose a simple finite-sample estimator and derive theoretical guarantees, including exponential error decay and logarithmic sample complexity in the number of objects (Theorem~\ref{thm:sample}).

We empirically validate the exponential error decay and logarithmic sample complexity using synthetic SCMs, where cascade failure probabilities can be directly and precisely controlled.
We then evaluate our method on a suite of chain-reaction environments exhibiting parallel, delayed, and intertwined interactions. Across all settings, our approach reliably and rapidly recovers the correct causal structure from a small number of interventional samples. In contrast, observational heuristics based on temporal order or collision detection fail precisely in regimes where causal ambiguity arises, such as near-simultaneous events.


\paragraph{Contributions.}
We summarize our contributions as follows:
\begin{enumerate}
    \item \textbf{Problem formulation.}
    We formalize causal discovery in chain-reaction systems via a simple causal abstraction: representing physical interactions with binary object-activation variables and a directed tree graph encoding \emph{causal responsibility}.

    \item \textbf{Structural identifiability from blocking interventions.}
    Under a monotone cascade structural causal model, we prove that the causal graph is uniquely identifiable from \emph{single-object blocking interventions}, with ancestor--descendant relations revealed directly by interventional activation probabilities.
    
    \item \textbf{Finite-sample estimation with theoretical guarantees.}
    We introduce a minimal estimator for recovering the graph under single-object blocking interventions and prove exponential error decay with logarithmic sample complexity in the number of objects.

    \item \textbf{Evaluation on synthetic data and chain-reaction environments.}
    We validate the theory using synthetic causal models and also demonstrate reliable structure recovery in diverse chain-reaction environments, where observational heuristics based on temporal order or contact fail due to delayed or overlapping causal effects.
\end{enumerate}

\section{Related Work}

\paragraph{Causal discovery with interventions.}
A central result in causal inference is that purely observational data are, in general, insufficient to completely identify causal structure, and that interventions can reduce or eliminate ambiguities \citep{pearl2009causality}.
%
Classical constraint-based and score-based causal discovery methods, such as PC and GES \citep{spirtes2000causation,chickering2002optimal}, are defined for observational data and identify a Markov equivalence class. When interventional data are available, targeted manipulations can shrink these equivalence classes and make additional edge directions identifiable \citep{eberhardt2006n,hauser2012characterization}.
Building on this insight, subsequent work has developed causal discovery methods that explicitly leverage interventional data\textemdash{}either with known or unknown intervention targets\textemdash{}to further orient edges beyond what is identifiable from observations alone \citep{wang2017permutation,mooij2020joint,brouillard2020differentiable}. These methods aim for broad applicability across large classes of structural causal models. As a result, they may yield partial identifiability or require substantial data to resolve ambiguity, depending on the intervention regime and modeling assumptions. 
In contrast, our work focuses on a restricted but meaningful class of systems in which blocking interventions induce deterministic cascade patterns. This structure enables simple, provably identifiable recovery of the full causal graph with explicit finite-sample guarantees.

\paragraph{Causal discovery under structural and invariance assumptions.}
Another line of work makes additional structural assumptions to make causal structure identifiable. Invariant causal prediction \citep{peters2015causal} exploits the stability of causal mechanisms across environments to identify causal parents, while other approaches leverage functional restrictions such as linearity with non-Gaussian noise \citep{shimizu2006linear}, additive noise models \citep{hoyer2008nonlinear}, or monotonicity and ordering assumptions to simplify structure learning. Our work follows this philosophy by imposing a simple and interpretable structural assumption on how failures and activations propagate in chain-reaction systems, enabling exact identifiability and an efficient estimator from interventional data.

\paragraph{Deterministic and cascade-style propagation models.}
Deterministic and near-deterministic propagation models have been studied in several related domains, including Boolean networks \citep{kauffman1969metabolic, shmulevich2002probabilistic}, reliability theory and fault-tree analysis \citep{haasl1981fault}, and models of failure propagation in engineered systems \citep{leveson2016engineering}.
These models capture systems in which downstream behavior is tightly constrained by upstream component states, often with stochastic noise at individual nodes.
While these models are not typically framed in the language of causal discovery, they share the key structural property that downstream behavior is constrained by upstream states.

\section{Problem Setup}

\subsection{Preliminaries}

A structural causal model consists of a set of variables $X = (X_1,\dots,X_N)$, a directed acyclic graph $G^\star$ encoding causal relationships,
and a collection of structural equations $X_j \coloneqq f_j(X_{\Pa(j)}, U_j)$, where $\Pa(j)$ denotes the parents of node $j$ in $G^\star$ and $U_j$ are exogenous noise variables.
An \emph{intervention} $\doop{X_i = x}$ replaces the structural equation for $X_i$ with $X_i \equiv x$, thereby cutting all incoming edges into $X_i$ while leaving the rest of the mechanisms unchanged \citep{pearl2009causality}.
In this work, we consider \emph{single-node interventions} and reason about causal effects through interventional distributions of the form $\Pr(X_j \mid \doop{X_i = x})$.
We write $\Desc(i)$ and $\Anc(i)$ for the sets of descendants and ancestors of node $i$ in $G^\star$, respectively.

\subsection{Chain-Reaction Causal Model}

We model a Rube Goldberg--style machine as a collection of $N$ objects, each associated with a binary random variable on each execution
\[
X_j \in \{0,1\}, \qquad j = 1,\dots,N,
\]
where $X_j = 1$ indicates that object $j$ becomes \emph{active} (e.g., moves, topples, or is pressed) at some point during the execution, and $X_j = 0$ indicates that it remains inactive.
The causal structure is represented by an unknown directed graph $G^\star = (V,E^\star)$ with $V = \{1,\dots,N\}$. A directed edge $i \to j$ means that object $i$ is \emph{responsible for triggering} the activation of object $j$ as part of the chain reaction.

We focus on systems that implement a directed cascade: each object has at most one direct trigger, but may trigger multiple downstream objects. Accordingly, we assume that $G^\star$ is a directed tree with a single root node. This abstraction captures the functional structure of Rube Goldberg machines, in which effects propagate along a designed sequence of components rather than through symmetric physical interactions.

\subsection{Monotone Cascade Structural Causal Model}

The activation variables evolve according to a \emph{monotone cascade} SCM.
For each node $j$, the structural equation is
\begin{equation}
X_j =
\begin{cases}
0, & \text{if } \exists\, p \in \Pa(j) \text{ with } X_p = 0, \\[4pt]
Z_j, & \text{otherwise},
\end{cases}
\label{eq:monotone}
\end{equation}
where $Z_j$ are exogenous Bernoulli noise variables with non-zero success probability.
Intuitively, $Z_j$ captures stochastic failures in physical execution, such as small misalignments when resetting the machine (e.g., imperfectly spaced dominoes or slightly displaced objects). Thus, an object may fail to activate spontaneously even if all required upstream triggers are active. However, if any required upstream trigger fails, activation of all downstream objects is deterministically prevented.

This SCM reflects key properties of the chain-reaction systems we consider:
(i) \emph{asymmetry of responsibility}, induced by the tree structure, where each object has a unique upstream trigger and causal influence flows in a well-defined upstream--downstream direction (ii) \emph{monotonicity}, since upstream activations can enable downstream activations, but never inhibit them, and (iii) \emph{cascade-style propagation}, since activations propagate sequentially along directed paths and blocking an upstream object deterministically suppresses all downstream activations.

\subsection{Interventions and Data Collection}\label{sec:dataset}

We observe the system through repeated executions.
In each execution $e$, we perform either (i) an observational run with no intervention, or (ii) a \emph{blocking intervention} on a single object.
A blocking intervention on object $i$ is modeled as \[\doop{X_i = 0}\] which prevents object $i$ from activating while leaving the rest of the system unchanged.
Operationally, this corresponds to physically holding an object in place and allowing the remainder of the machine to run. Under the monotone cascade model~\eqref{eq:monotone}, this intervention deterministically forces all descendants of $i$ to remain inactive, while non-descendants behave according to their observational distribution.

Formally, the interventional dataset consists of samples
\[
\{(I_e, X^{(e)})\}_{e=1}^M,
\]
where $I_e \in \{1,\dots,N\}$ denotes the intervened object in episode $e$
(and $I_e = \varnothing$ indicates no intervention),
and $X^{(e)} \in \{0,1\}^N$ is the observed activation vector at the end of the episode.
For example, for Figure~\ref{fig:intuition}, a possible dataset is $\{
(I_e=\varnothing,\, X^{(e)}=(1,1,1,1)),
(I_e=\varnothing,\, X^{(e)}=(0,1,0,1)),
(I_e=2,\, X^{(e)}=(0,0,0,1)),
(I_e=4,\, X^{(e)}=(0,0,0,0))
\}$ where blocking object $i$ deterministically suppresses all of its descendants.
A larger example is provided in Appendix~\ref{app:dataset}.

\subsection{Learning Objective}

Our goal is to recover the true causal graph $G^\star$ governing the chain-reaction system from interventional data. Given repeated executions under blocking interventions, we seek to identify the directed structure that encodes causal responsibility between objects.
Crucially, we do not assume access to temporal orderings, collision events,
or detailed physical state trajectories. All information is extracted from binary activation outcomes under interventions. In the following section, we show that under the monotone cascade model, single-object blocking interventions are sufficient to uniquely identify the full causal structure, and we provide a simple estimator with finite-sample guarantees.


\section{Method}\label{sec:method}

We now present our causal discovery method for chain-reaction systems.
We first establish identifiability of the causal structure and then describe a finite-sample estimator, the reconstruction algorithm, and provide theoretical guarantees. 

\subsection{Identifiability from Blocking Interventions}

For distinct objects $i \neq j$, define the \emph{interventional activation probability}
\begin{equation}
p_{ij} \coloneqq \Pr\!\bigl( X_j = 1 \mid \doop{X_i = 0} \bigr).
\label{eq:pij}
\end{equation}
Intuitively, $p_{ij}$ measures whether object $j$ can still become active
when object $i$ is blocked. The montone cascade model allows the following characterization of descendant relations.

\begin{lemma}[Deterministic cascade relation]\label{lem:cascade}
For any $i \neq j$,
\[
j \in \Desc(i)
\quad \Longleftrightarrow \quad
p_{ij} = 0.
\]
\end{lemma}

\begin{proof}
If $j$ is a descendant of $i$, then blocking $i$ forces every object on every directed path
$i \rightsquigarrow j$ to remain inactive under~\eqref{eq:monotone},
hence $X_j = 0$ and $p_{ij}=0$.
Conversely, if $j$ is not a descendant of $i$, then the structural equation for $X_j$
does not depend on $X_i$.
Since $Z_j$ has non-zero success probability, $X_j = 1$ remains possible under $\doop{X_i = 0}$,
implying $p_{ij} > 0$.
\end{proof}


Lemma~\ref{lem:cascade} shows that single-object blocking interventions reveal the full
ancestor--descendant relation.
We now define the \emph{ancestor matrix}
\[
A(i,j) \coloneqq \mathbbm{1}\{ p_{ij} = 0 \}.
\]

Because the true causal graph $G^\star$ is a directed tree, each node has a unique parent / closest ancestor. Once all ancestor--descendant relations are identified, the direct triggering structure (the directed tree) can be recovered by \emph{transitive reduction}.

\begin{theorem}[Identifiability]\label{thm:ident}
Under the monotone cascade model and single-object blocking interventions,
the true causal tree $G^\star$ is uniquely identifiable from the interventional
activation probabilities $\{p_{ij}\}_{i,j}$.
\end{theorem}

\begin{proof}
Lemma~\ref{lem:cascade} shows that the interventional probabilities $\{p_{ij}\}$ uniquely determine the ancestor--descendant relation. 
Since the true causal graph $G^\star$ is a directed tree, this relation admits a unique transitive reduction, which coincides with $G^\star$ \citep{aho1972transitive}.
\end{proof}


\subsection{Estimation from a Finite Interventional Dataset}

Let $n_i$ denote the number of executions in which object $i$ is blocked.
For each ordered pair $(i,j)$, we estimate the interventional activation probability $p_{ij}$ by the empirical frequency
\begin{equation}
\hat p_{ij}
\;=\;
\frac{1}{n_i}
\sum_{e: I_e = i}
\mathbbm{1}\{ X_j^{(e)} = 1 \}.
\label{eq:phat}
\end{equation}

Under the monotone cascade model, descendants of $i$ are deterministically inactive under $\doop{X_i = 0}$. Therefore, we classify $j$ as a descendant of $i$ whenever $\hat p_{ij} = 0$. Conversely, if object $j$ is observed to activate at least once when $i$ is blocked, we conclude that $j \notin \Desc(i)$. This yields an empirical ancestor matrix
\[
\hat A(i,j) \coloneqq \mathbbm{1}\{ \hat p_{ij} = 0 \}.
\]

\paragraph{Behavior in the low-sample regime.}
When the number of interventions $n_i$ is small, it is possible that $\hat p_{ij}=0$ for many non-descendant pairs. In this regime, the estimator will overestimate the ancestor relation. However, note that the resulting errors are one-sided: true descendants are never misclassified as non-descendants under the monotone cascade model, and false positives vanish as $n_i$ grows. As our theoretical guarantees show in Section~\ref{sec:theoretical}, the probability of such spurious ancestor relations decays exponentially in $n_i$, and the true ancestor matrix is recovered with high probability once $n_i$ exceeds a logarithmic threshold.



\subsection{Reconstruction Algorithm}

Given the estimated ancestor matrix $\hat A$, we reconstruct the causal graph by enforcing acyclicity and computing a transitive reduction.
In the low-sample regime, we can have spurious ancestor relations including 2-cycles when both $\hat p_{ij}$ and $\hat p_{ji}$ equal zero.
If observational episodes are available, such cycles can sometimes be pruned using the monotone implication that $j \in \Desc(i)$ would forbid observing $X_j = 1$ while $X_i = 0$, when such events are observed under stochastic failures.
Any remaining cycles are broken deterministically (e.g., by index) to obtain a DAG prior to transitive reduction. These heuristics are not required for identifiability and only affect the low-sample regime. As the number of interventions grows, spurious relations vanish exponentially (Theorem~\ref{thm:sample}). Algorithm~\ref{alg:tree} summarizes the reconstruction.

\begin{algorithm}[t]
\caption{Cascade Tree Reconstruction}
\label{alg:tree}
\begin{algorithmic}[1]
\REQUIRE Interventional dataset $\{(I_e, X^{(e)})\}_{e=1}^M$
\STATE Compute empirical probabilities $\hat p_{ij}$ using~\eqref{eq:phat}
\STATE Construct ancestor matrix $\hat A(i,j) = \mathbbm{1}\{\hat p_{ij} = 0\}$
\STATE Enforce acyclicity of $\hat A$ by breaking directed cycles
\STATE Compute the transitive reduction of $\hat A$
\STATE \textbf{Return} reconstructed causal graph $\hat G$
\end{algorithmic}
\end{algorithm}


\subsection{Theoretical Guarantees}\label{sec:theoretical}

We now quantify the finite-sample behavior of the estimator.
Define the smallest non-descendant activation probability as follows
\[
q_{\min}
\;\coloneqq\;
\min_{i \neq j:\, j \notin \Desc(i)}
\Pr\!\bigl( X_j = 1 \mid \doop{X_i = 0} \bigr),
\qquad q_{\min} > 0,
\]

\begin{theorem}[Sample Complexity]\label{thm:sample}
If $j \notin \Desc(i)$, then
\begin{equation}
\Pr\!\bigl( \hat A(i,j) = 1 \bigr)
\;\le\;
\exp(- q_{\min} n_i).
\label{eq:expbound}
\end{equation}
That is, the probability of a false positive ancestor relation
decays exponentially in the number of interventions $n_i$. Moreover,
\begin{equation}\label{eq:probcorrect}
\Pr(\hat A = A)
\;\ge\;
1 - N(N-1)\, \exp(- q_{\min} n_{\min}),
\end{equation}
where $n_{\min} = \min_i n_i$.
Thus, if
\begin{equation*}\label{eq:samplebound}
n_{\min}
\;\ge\;
\frac{1}{q_{\min}}
\Bigl( \log(N (N - 1)) + \log \tfrac{1}{\delta} \Bigr),
\end{equation*}
then $\hat A = A$ with probability at least $1-\delta$.
\end{theorem}

We provide the proof in Appendix~\ref{app:proofs}.

\begin{corollary}[Exact Tree Recovery]
If $\hat A = A$, then the transitive reduction of $\hat A$ equals $G^\star$.
In particular,
\[
\Pr(\hat G = G^\star) \;\ge\; 1 - \delta.
\]
\end{corollary}

Inequality~\eqref{eq:expbound} quantifies how quickly spurious ancestor relations ($j \notin \Desc(i)$) disappear as more interventions are collected.
For a fixed non-descendant pair $(i,j)$, the only way $j$ can be incorrectly classified as downstream of $i$ is if it never activates in any of the $n_i$ trials under $\doop{X_i=0}$.
Since $j$ activates with probability at least $q_{\min}$ in each such trial, this event becomes exponentially unlikely in $n_i$.
Inequality~\eqref{eq:probcorrect} shows that, once each object is blocked a logarithmic number of times in the number of objects, with high probability no spurious ancestor relations remain and the estimated and true ancestor matrices coincide.





\section{Experiments}\label{sec:experiments}

We evaluate our method with two goals:
(i) to empirically validate the finite-sample guarantees established in Section~\ref{sec:theoretical},
and (ii) to demonstrate reliable causal structure recovery in chain-reaction environments inspired by Rube Goldberg machines. All experiments\footnote{Code is available at \url{https://github.com/panispani/chain-reaction-causal-discovery}} were conducted on a shared CPU server (AMD EPYC 9454, 96 cores, 192 threads, 1.5~TB RAM).
Complete experimental results are reported in Appendix~\ref{app:results}; here we present the most informative findings.

\subsection{Environments}

We evaluate on six chain-reaction environments (Figure~\ref{fig:envs}), each consisting of interacting objects simulated with the Pymunk physics engine, together with buttons that trigger delayed, non-contact effects by opening platforms (i.e., removing wall segments). To model imperfect setup and physical variability, we introduce a displacement parameter $\Delta$: at the start of each episode, the position of every object is independently perturbed by a uniform displacement in $[-\Delta,\Delta]$. These perturbations induce stochastic execution failures and correspond to the exogenous noise variables $Z_j$ in the monotone cascade SCM.

Each episode consists of a single \emph{blocking intervention}. Interventions are performed in rounds: in each round, every object is intervened on exactly once, in a random order. The intervened object is held fixed by disabling its physics interactions, while the rest of the system evolves unchanged. 
For each episode, we record only the final binary activation vector indicating which objects became active. From observation alone, the causal structure of these environments is generally ambiguous. For example, in the Slot-machine and Parallel Trigger environments, multiple buttons may be pressed simultaneously, releasing different downstream objects at the same time. Appendix~\ref{app:envs} provides worked-out execution rollouts illustrating these mechanisms.




\begin{figure}[t]
  \centering
  \footnotesize

  \subfigure[\textbf{Minimal Chain}]{%
    \includegraphics[
      width=0.32\linewidth,
      trim=0 40 200 0,
      clip
    ]{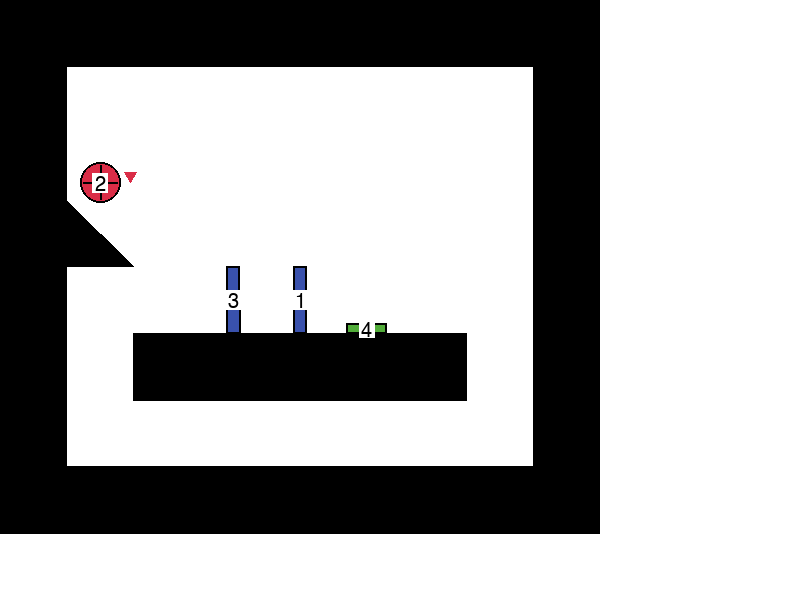}
  }\hfill
  \subfigure[\textbf{Sequential Chain}]{%
    \includegraphics[
      width=0.32\linewidth,
      trim=0 40 480 0,
      clip
    ]{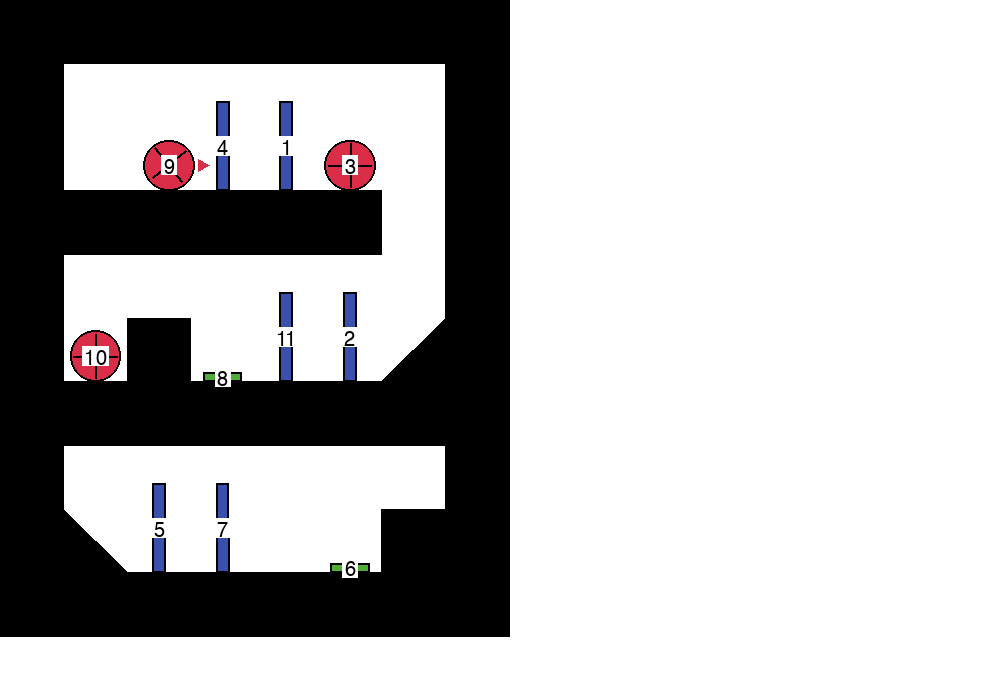}
  }\hfill
  \subfigure[\textbf{Parallel Triggers}]{%
    \includegraphics[
      width=0.32\linewidth,
      trim=0 50 480 0,
      clip
    ]{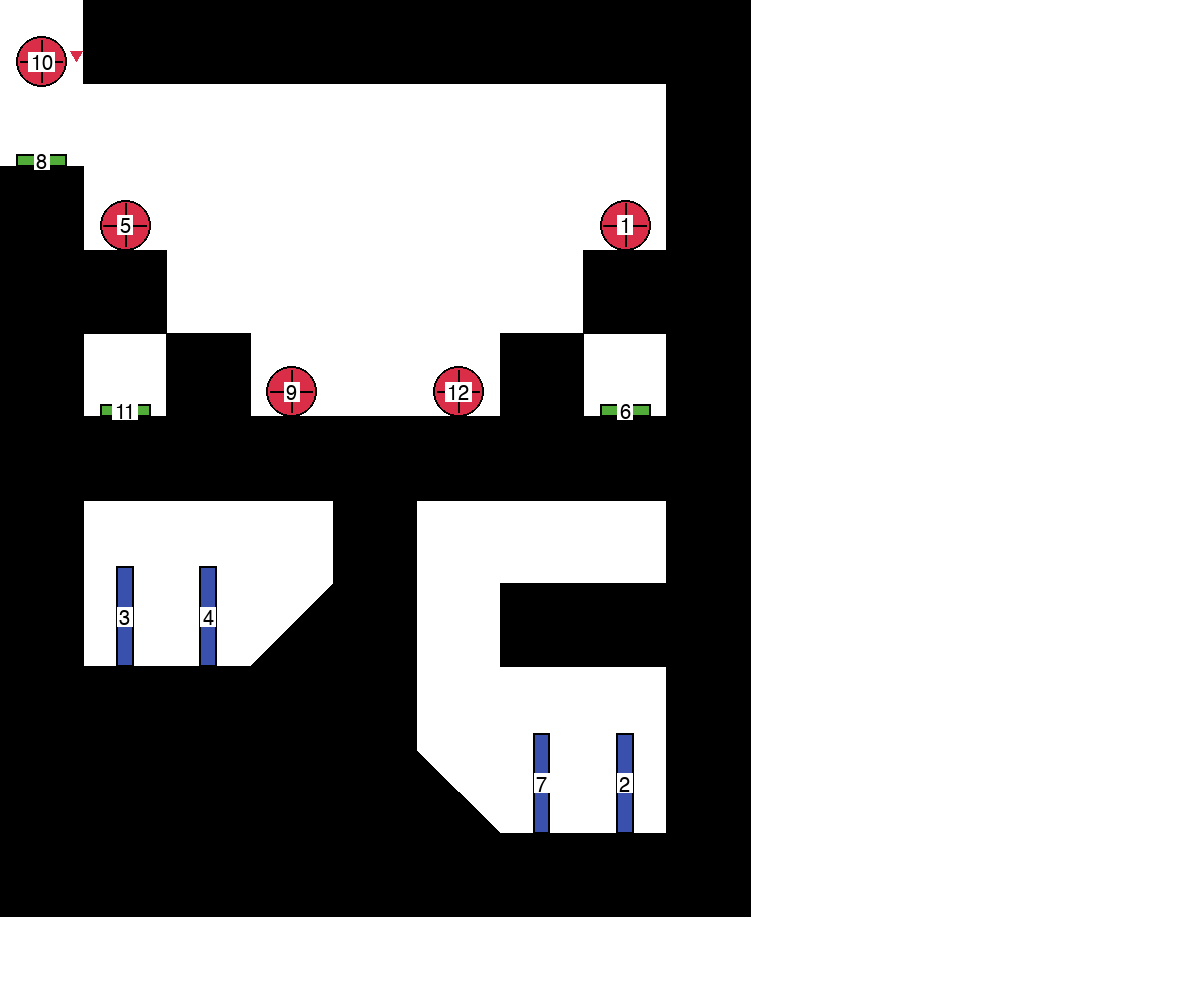}
  }

  \vspace{0.5em}

  \subfigure[\textbf{Intertwined Mechanisms}]{%
    \includegraphics[
      width=0.32\linewidth,
      trim=0 50 450 0,
      clip
    ]{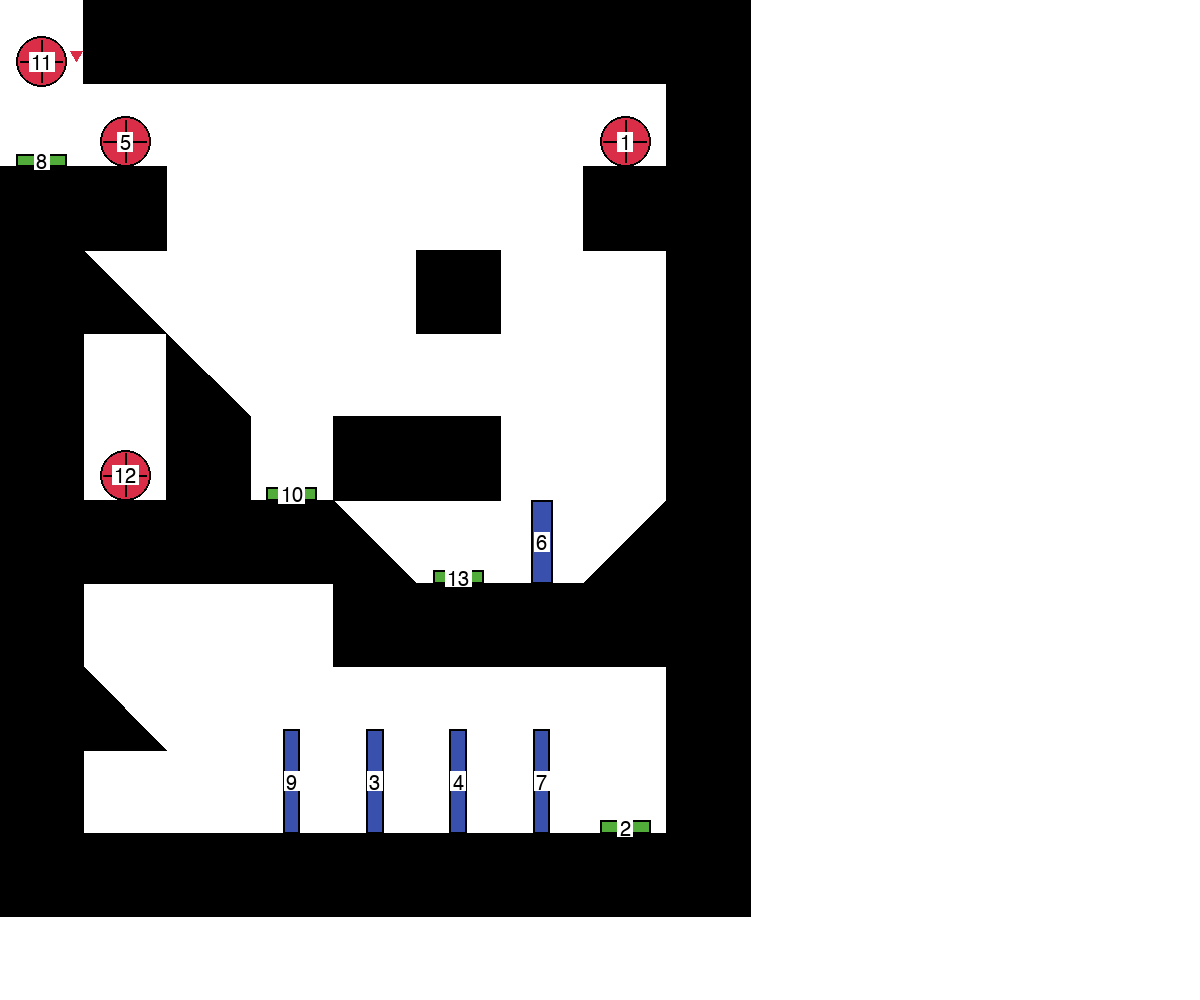}
  }\hfill
  \subfigure[\textbf{Linear Slot-Machine}]{%
    \includegraphics[
      width=0.32\linewidth,
      trim=0 60 330 0,
      clip
    ]{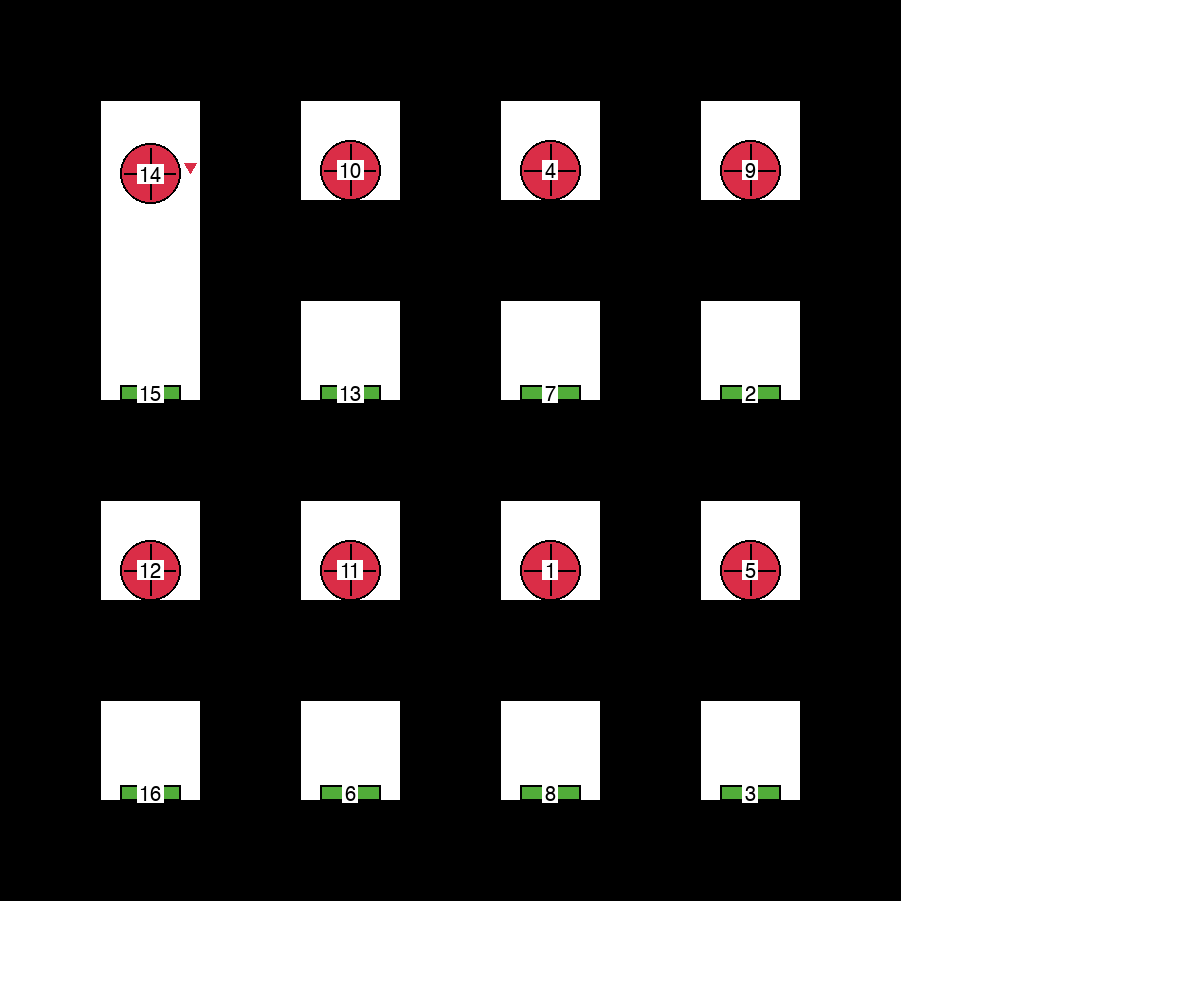}
  }\hfill
  \subfigure[\textbf{Large Slot-Machine}]{%
    \includegraphics[
      width=0.32\linewidth,
      trim=0 40 570 0,
      clip
    ]{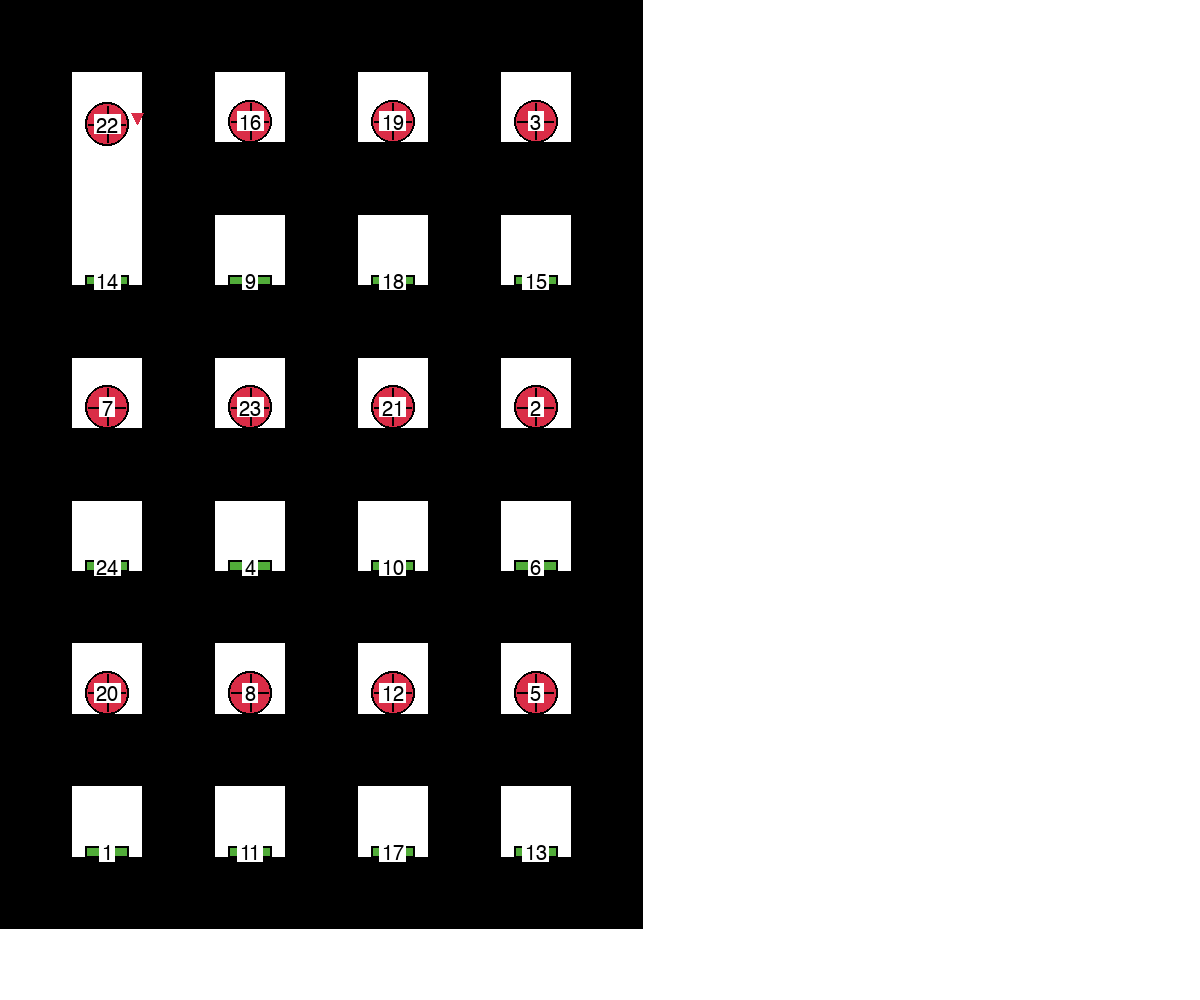}
  }

  \caption{
  Collection of environments used for evaluation.
  }
  \label{fig:envs}
\end{figure}

The six environments are:
\begin{itemize}
    \item \textbf{Minimal Chain.} A short linear cascade with four objects.
    \item \textbf{Sequential Chain.} A longer linear system with eleven objects.
    \item \textbf{Parallel Triggers.} Parallel activations cause simultaneous downstream effects.
    \item \textbf{Intertwined Mechanisms.} Concurrent interactions occur in different regions of the environment but work jointly for the whole system.
    \item \textbf{Linear Slot-Machine.} Cascades with non-contact effects via button releases.
    \item \textbf{Large Slot-Machine.} A larger Slot-machine variant with multiple concurrent causal branches.
\end{itemize}

\subsection{Baselines}

We compare against two observational heuristic baselines, each given privileged access to
perfect collision detection and exact activation times.
\emph{Collision-as-influence} adds an edge $i\!\to\!j$ whenever object $i$ collides with an inactive object $j$.
\emph{Temporal precedence} assigns edges based on activation order, attributing non-contact activations to the most recent collision event. Both baselines operate purely on observational data and are included to show the limits of even oracle-level perception in the absence of interventions. We also include results for the PC algorithm~\citep{spirtes2000causation}.

\subsection{Validating the finite sample guarantees}

Theoretical analysis predicts exponential decay of error with the number of blocking interventions per object. In the physical environments, we control execution noise indirectly via the displacement parameter $\Delta$.
However, $\Delta$ cannot be increased arbitrarily without causing objects to collide or overlap at initialization, which would invalidate the mechanism (e.g., adjacent objects are separated by $\approx0.6$ units in some environments, corresponding to a maximum displacement of $0.3$).

Figure~\ref{fig:trend1} plots skeleton SHD and recovery probability as a function of the number of interventions per object across all environments, using the largest feasible displacement for each. Since these environments are typically solved with few samples with our method, we additionally construct synthetic SCMs using the causal graphs of "Parallel Triggers" and "Large Slot-Machine" and the monotone cascade equations~\eqref{eq:monotone}. This allows us to directly control the failure probability of each variable.
Figure~\ref{fig:trend2} shows that in this controlled setting, the skeleton SHD decays exponentially with the number of interventions, and the probability of exact recovery increases exponentially (in agreement with Theorem~\ref{thm:sample}).

\begin{figure}[h]
  \centering
  \footnotesize

  \subfigure[\label{fig:envstop2}]{%
    \includegraphics[width=0.45\linewidth]{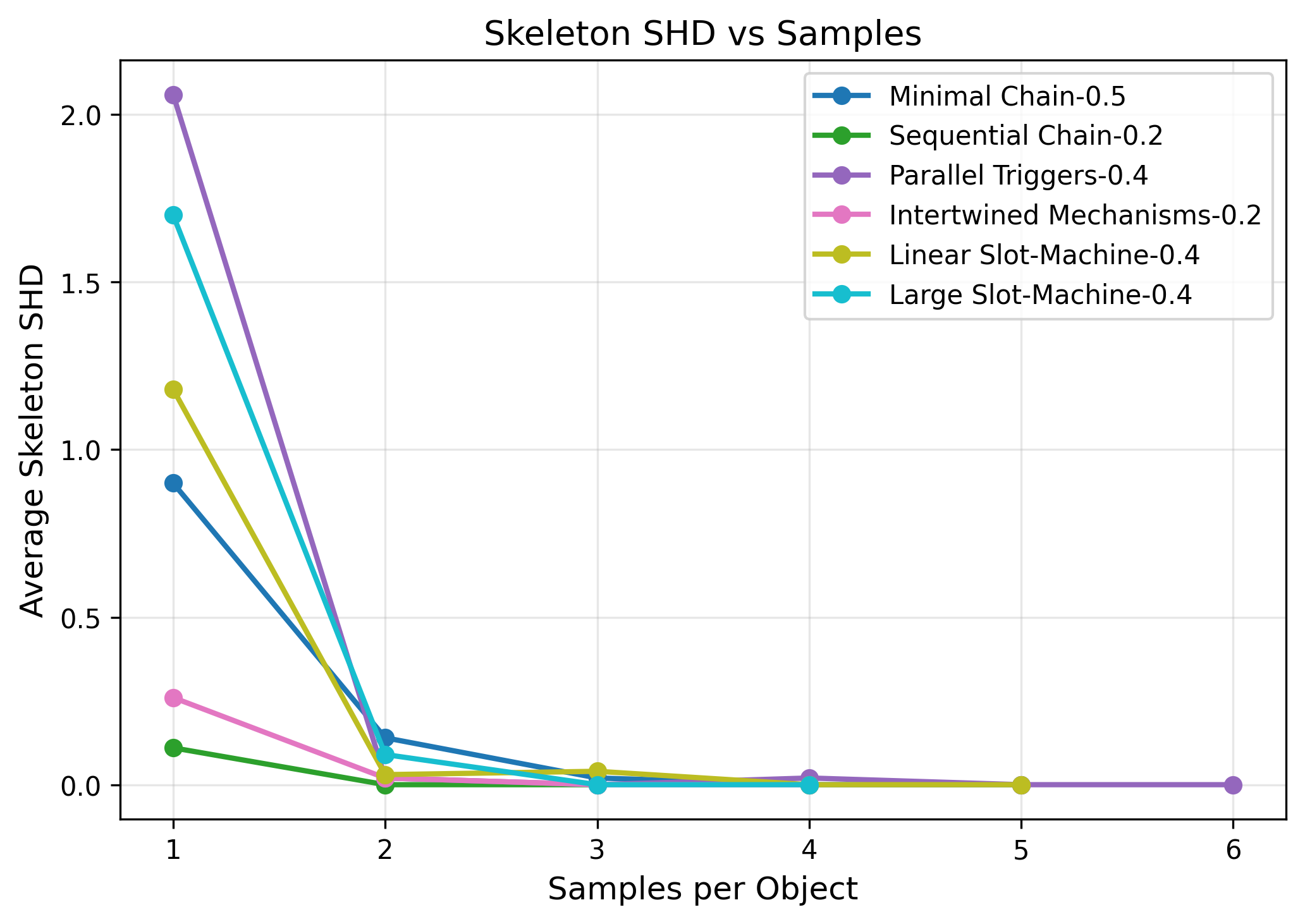}
  }\hfill%
  \subfigure[\label{fig:envsbottom2}]{%
    \includegraphics[width=0.45\linewidth]{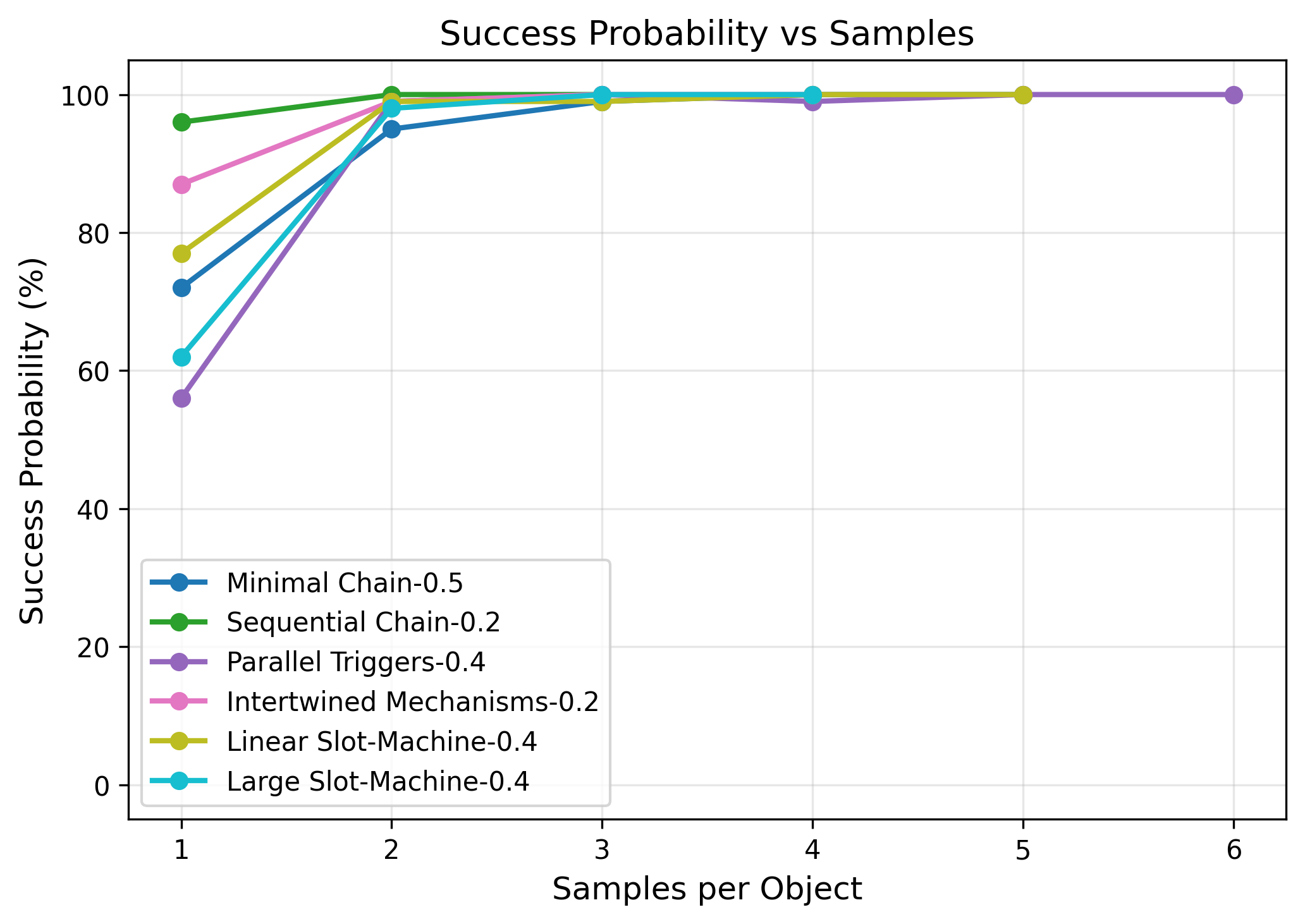}
  }

\caption{
Scaling with the number of blocking interventions per object.
(\emph{Left}) Average skeleton SHD as a function of the number of interventions.
(\emph{Right}) Probability of exact recovery.
Each curve corresponds to one environment evaluated at its largest feasible displacement $\Delta$.
We see strong sample efficiency.
}
  \label{fig:trend1}
\end{figure}


\begin{figure}[h]
  \centering
  \footnotesize

  \subfigure[\label{fig:envstop3}]{%
    \includegraphics[width=0.45\linewidth]{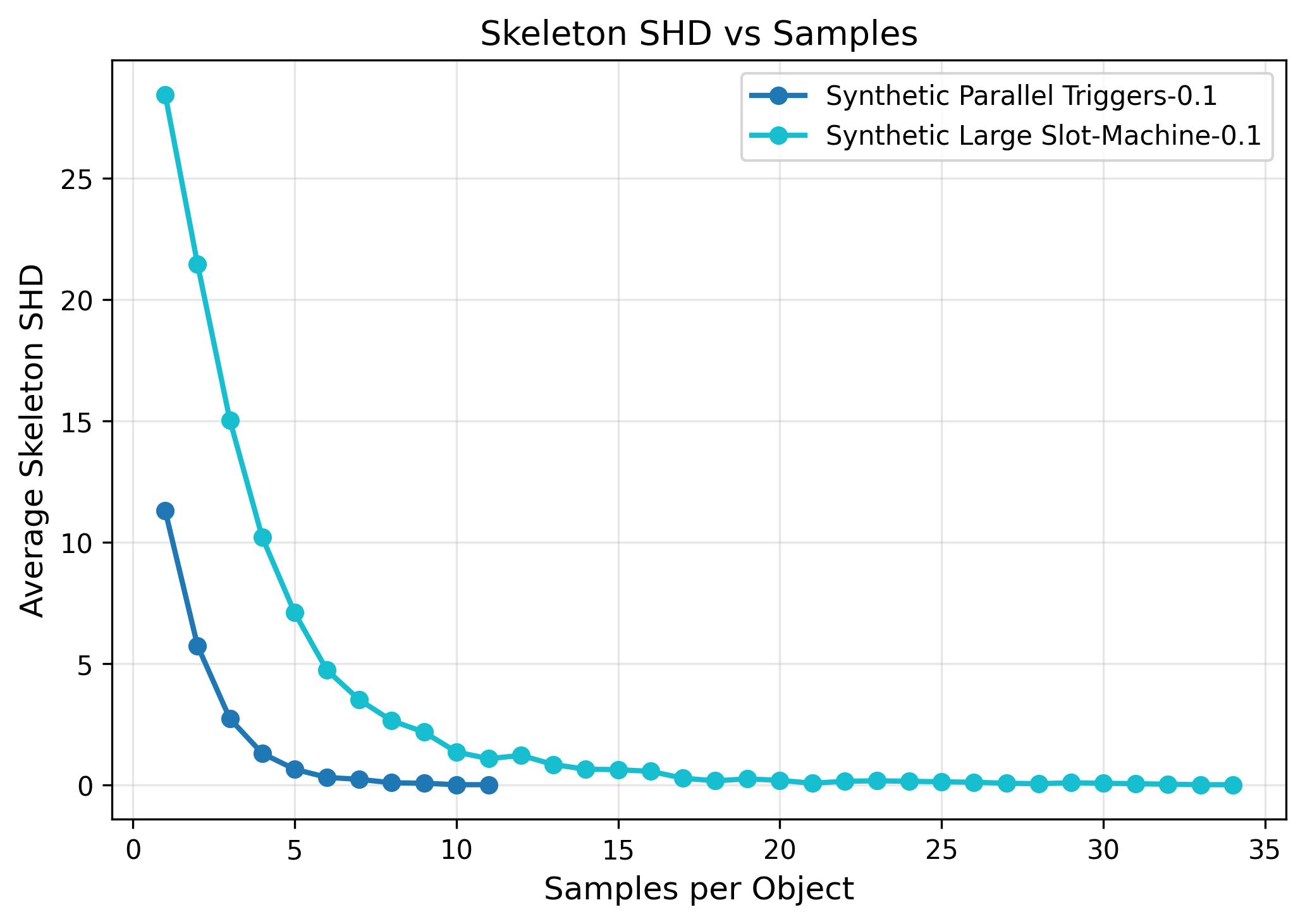}
  }\hfill%
  \subfigure[\label{fig:envsbottom3}]{%
    \includegraphics[width=0.45\linewidth]{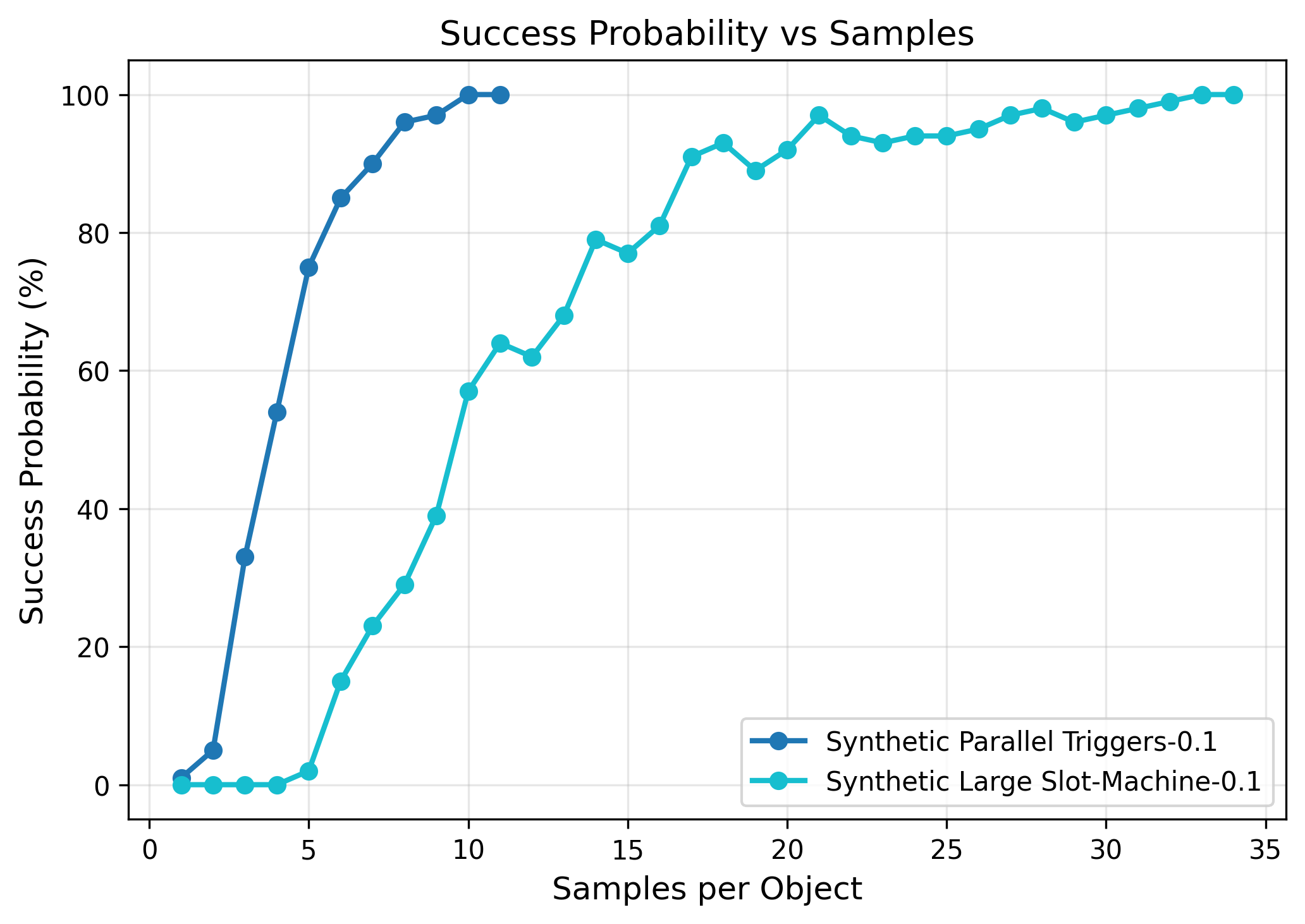}
  }

  \caption{
  Scaling with the number of blocking interventions per object on synthetic SCMs
(\emph{Left}) Average skeleton SHD.
(\emph{Right}) Probability of exact recovery.
In this controlled setting, skeleton SHD decays exponentially, and recovery probability increases exponentially with the number of interventions, in agreement with the finite-sample guarantees of Theorem~\ref{thm:sample}. 
 The number suffix denotes the Bernoulli failure parameter.}
  \label{fig:trend2}
\end{figure}

\subsection{Performance on Rube Goldberg--Inspired Environments}

We evaluate performance across displacement levels by progressively increasing the number of blocking interventions per object. For each environment and displacement $\Delta$, we identify the minimum number of interventions per object required for our method to achieve at least $95\%$ exact recovery over 100 random seeds,
where exact recovery means $\hat G = G^\star$ after transitive reduction.
We denote this quantity by $M_{\min}$. Baselines are evaluated using the same number of episodes for a fair comparison. Throughout, our method relies exclusively on interventional data.

Table~\ref{tab:summary_max_disp} summarizes results at the largest displacement $\Delta_{\max}$
considered for each environment, corresponding to the noisiest setting in which the mechanism remains meaningful.
We report $M_{\min}$ together with directed-edge F1 score and skeleton structural Hamming distance (SSHD, lower is better), both for our method and for the strongest observational baseline at $\Delta_{\max}$.

\begin{table}[t]
\centering
\small
\begin{tabular}{lccccc}
\toprule
Environment & $N$ & $\Delta_{\max}$ & $M_{\min}$ &
Our (F1 / Skel.\ SHD) & Best baseline (F1 / Skel.\ SHD) \\
\midrule
Minimal Chain          & 4  & 0.5 & 2 & 0.963 / 0.14 & 0.825 / 0.52 \\
Sequential Chain       & 11 & 0.2 & 1 & 0.995 / 0.07 & 0.714 / 2.57 \\
Parallel Triggers      & 12 & 0.4 & 2 & 0.999 / 0.02 & 0.699 / 5.55 \\
Intertwined Mechanisms & 13 & 0.2 & 2 & 0.999 / 0.02 & 0.702 / 4.95 \\
Linear Slot-Machine    & 16 & 0.4 & 2 & 0.999 / 0.03 & 0.726 / 7.00 \\
Large Slot-Machine     & 24 & 0.4 & 2 & 0.998 / 0.09 & 0.686 / 11.00 \\
\bottomrule
\end{tabular}
\caption{
\textbf{Performance at maximal displacement.}
For each environment, we report the noisiest displacement $\Delta_{\max}$ tested,
the minimum number of blocking interventions per object $M_{\min}$ for $\ge95\%$ exact recovery, and performance at that setting.
Our method consistently achieves near-perfect recovery using only 1--2 interventions per object, while observational heuristics have large structural errors.
Full sweeps across displacements and additional metrics are in Appendix~\ref{app:results}.
}
\label{tab:summary_max_disp}
\end{table}

\paragraph{Key takeaways.}
Two consistent patterns emerge. First, recovery is highly sample-efficient: across all environments, $M_{\min}\in\{1,2\}$ interventions per object suffice for reliable exact recovery, even under substantial execution noise.
Second, the gap to observational heuristics is qualitative rather than incremental: collision- and time-based methods exhibit large skeleton errors at $\Delta_{\max}$, reflecting fundamental causal ambiguity in the presence of parallel activations, delayed effects, and non-contact interactions.
These results support the central claim of the paper: when the causal abstraction aligns with mechanistic responsibility and interventions respect the cascade semantics, causal discovery can become both simple and reliable.

\section{Discussion and Limitations}


\paragraph{Observational identifiability under stochastic failures.}
In our monotone cascade model, purely observational data can asymptotically emulate blocking interventions when every component fails with positive probability, since conditioning on $X_i=0$ can reveal the same downstream suppression pattern as blocking $i$. 
But if some parts of the mechanism are deterministic, those informative "failure" events may never occur, so the causal structure can remain ambiguous even with infinite observational data (e.g., a ``Large Slot-Machine'' where $\Delta$ is small and balls always press the buttons).

\paragraph{Structural assumptions.}
Our method assumes that the underlying causal graph is a directed tree, so that each object has at most one direct trigger. This reflects the functional design of many chain-reaction systems, but excludes settings in which an object requires multiple parents to be activated (e.g., two objects must act together to activate a third).
Extensions to such systems require interventions on sets of objects, leading to a combinatorial increase in experimental complexity. Understanding potential trade-offs between feasible intervention design and causal expressivity is an important direction for future work.

\paragraph{Robustness to measurement noise.}
Our theoretical analysis assumes activation variables are observed without error.
In practice, activation must be inferred from noisy measurements,
and labels may be corrupted (e.g., objects mistakenly marked as active or inactive).
Our method extends to this setting by replacing the strict criterion
$\hat p_{ij} = 0$ with a threshold that accounts for finite samples and label noise. Robustness to measurement error is an important extension for real-world applications.

\paragraph{Active discovery.}

In our experiments, interventions are selected uniformly at random, and the outcomes of these interventions are collected into an offline dataset from which the causal graph is reconstructed. A natural next step is to pose causal discovery as a sequential decision-making problem, with the goal of minimizing the number of interventions needed to recover the graph with high confidence. This would require maintaining uncertainty over the current causal structure and prioritizing interventions that are maximally informative. This direction is particularly relevant in domains where interventions are costly, risky, or otherwise constrained.

\section*{Acknowledgements}

This research was supported by the UKRI Centre for Doctoral Training in Accountable, Responsible and Transparent AI (ART-AI) [EP/S023437/1].


\bibliography{clear}

\newpage
\appendix

\section{Worked-Out Execution Rollouts}
\label{app:envs}

This appendix provides qualitative visualizations of the chain-reaction environments used throughout the paper. For each environment, we include frame-by-frame execution rollouts illustrating how activations propagate through the system.

\paragraph{How to read the figures.}
Each figure is organized as a sequence of frames. Frames should be read from \emph{left to right} and then \emph{top to bottom}, similar to a comic strip.
Unless there is an intervention mentioned, it is an observational rollout.

\begin{figure}[H]
    \centering
    \includegraphics[trim=0 0 20 0,clip,width=\linewidth]{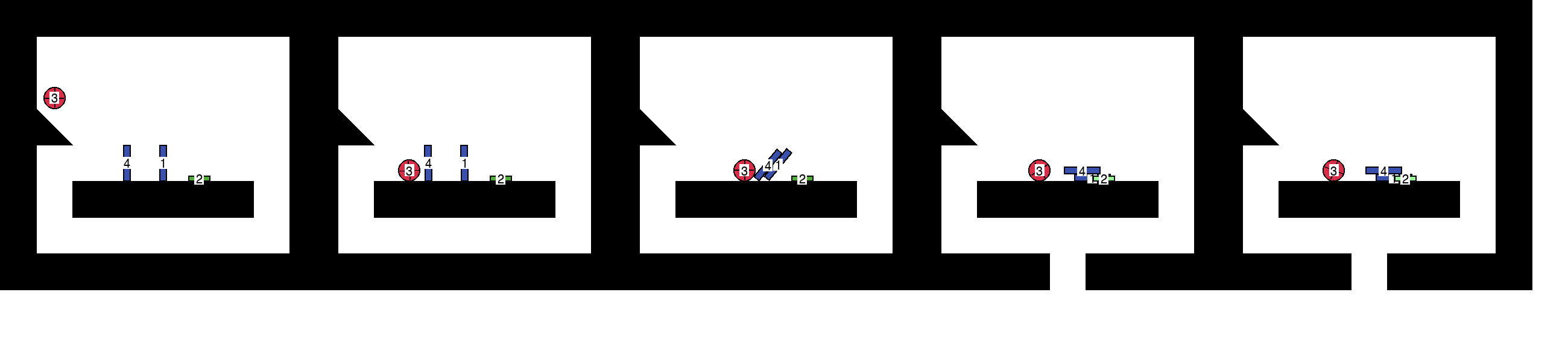}
    \caption{
    \textbf{Observational rollout of Minimal Chain}
    }\label{fig:env:t1}
\end{figure}

\begin{figure}[H]
    \centering
    \includegraphics[trim=0 0 20 0,clip,width=\linewidth]{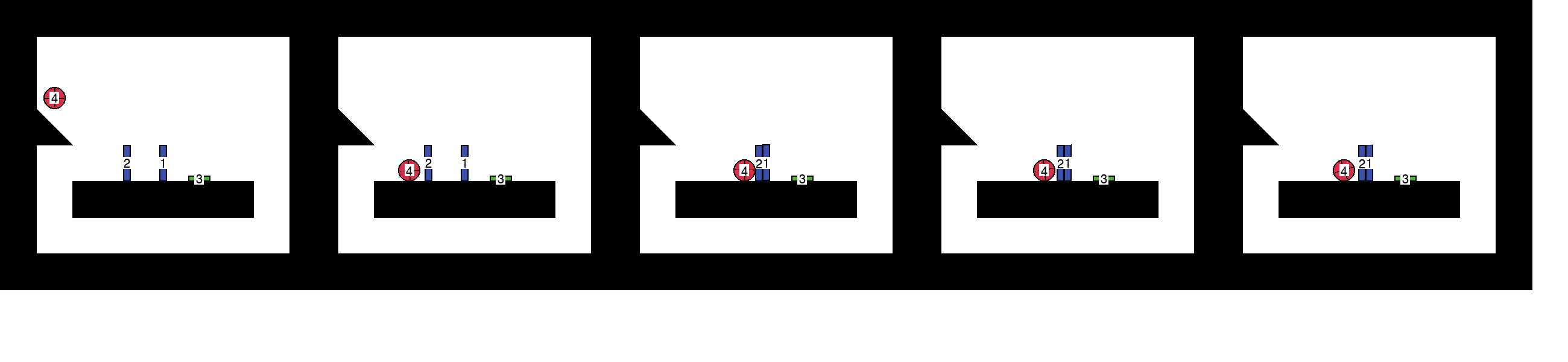}
    \caption{
    \textbf{Interventional rollout of Minimal Chain (intervention on object 1)}
    }
\end{figure}
\begin{figure}[H]
    \centering
    \includegraphics[trim=0 0 20 0,clip,width=\linewidth]{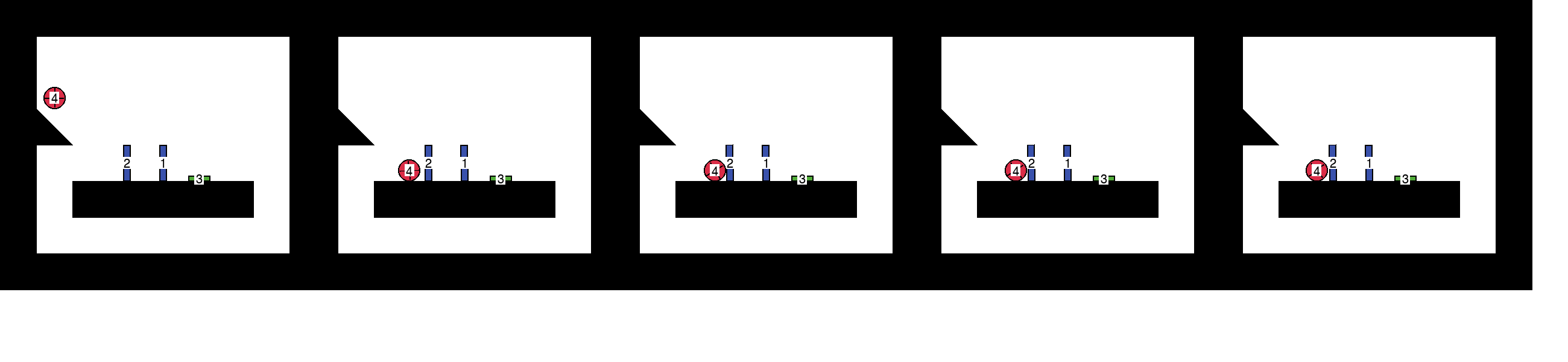}
    \caption{
    \textbf{Interventional rollout of Minimal Chain (intervention on object 2)}
    }
\end{figure}
\begin{figure}[H]
    \centering
    \includegraphics[trim=0 0 20 0,clip,width=\linewidth]{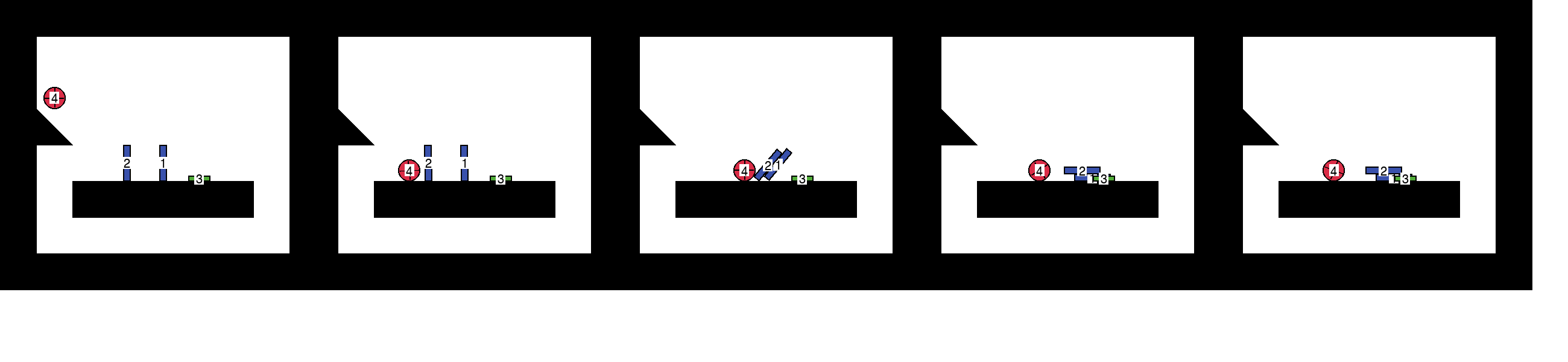}
    \caption{
    \textbf{Interventional rollout of Minimal Chain (intervention on object 3)}
    }
\end{figure}
\begin{figure}[H]
    \centering
    \includegraphics[trim=0 0 20 0,clip,width=\linewidth]{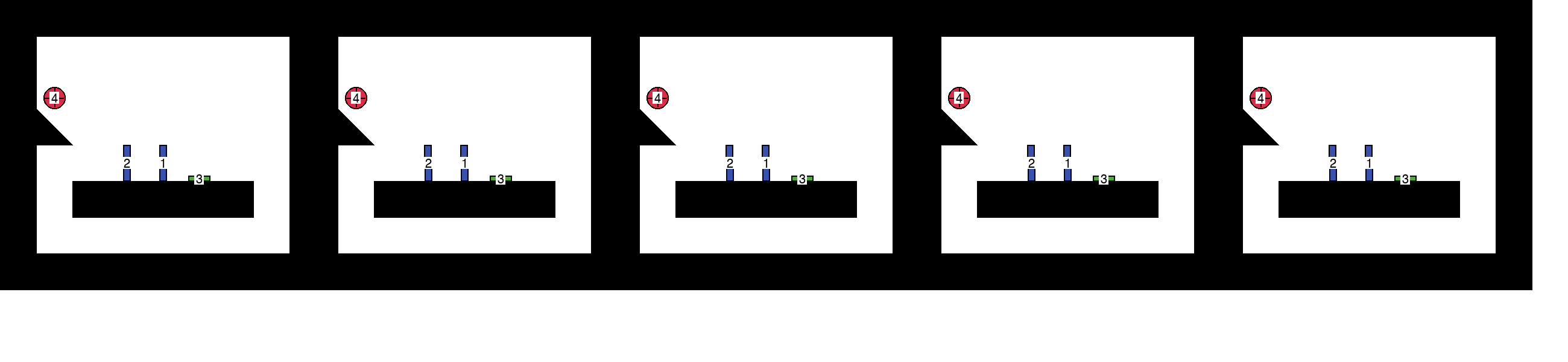}
    \caption{
    \textbf{Interventional rollout of Minimal Chain (intervention on object 4)}
    }
\end{figure}

\begin{figure}[H]
    \centering
    \includegraphics[trim=0 0 100 0,clip,width=\linewidth]{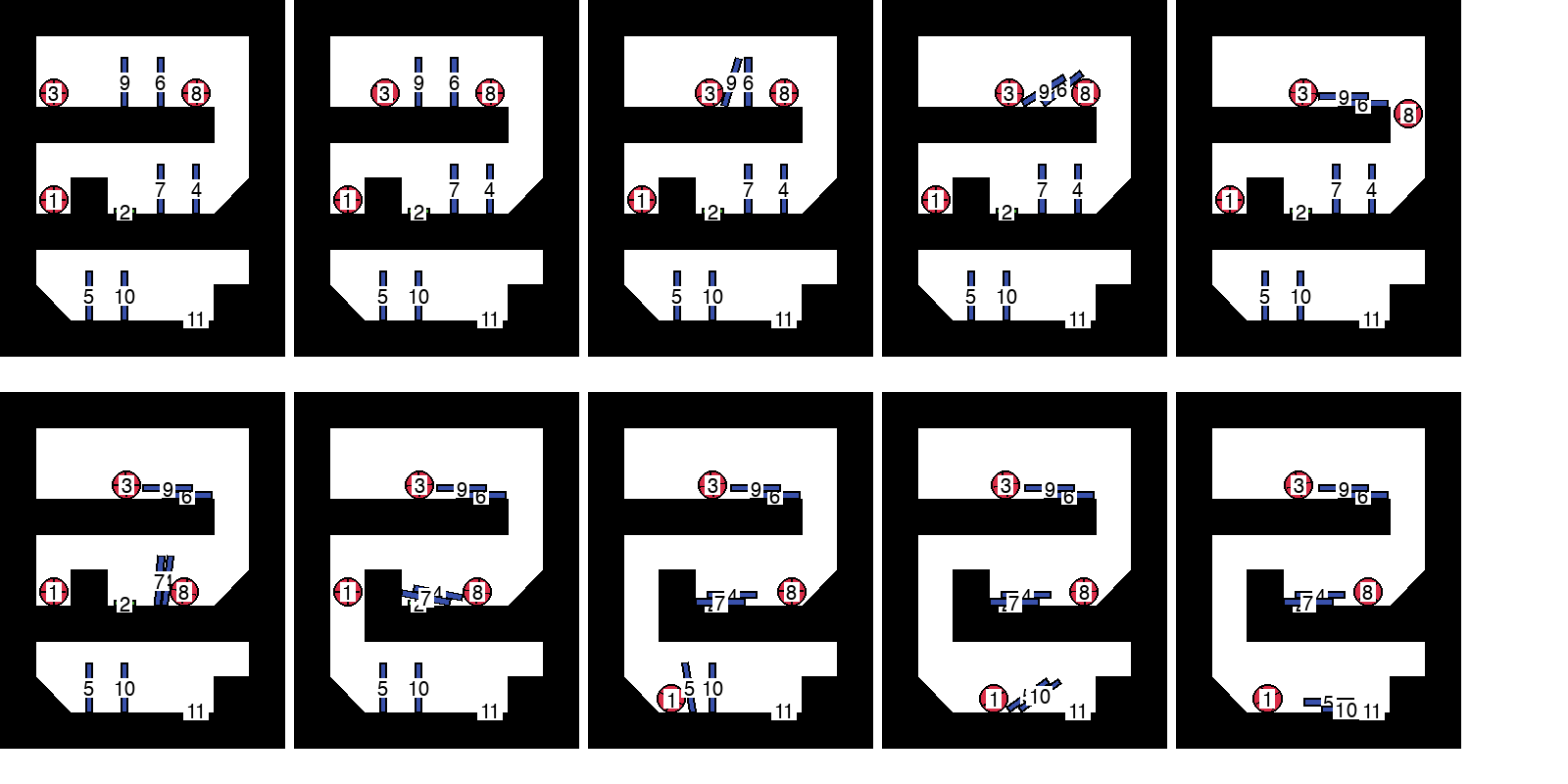}
    \caption{
    \textbf{Observational rollout of Sequential Chain}
    }\label{fig:env:t0}
\end{figure}

\begin{figure}[H]
    \centering
    \includegraphics[trim=0 0 100 0,clip,width=\linewidth]{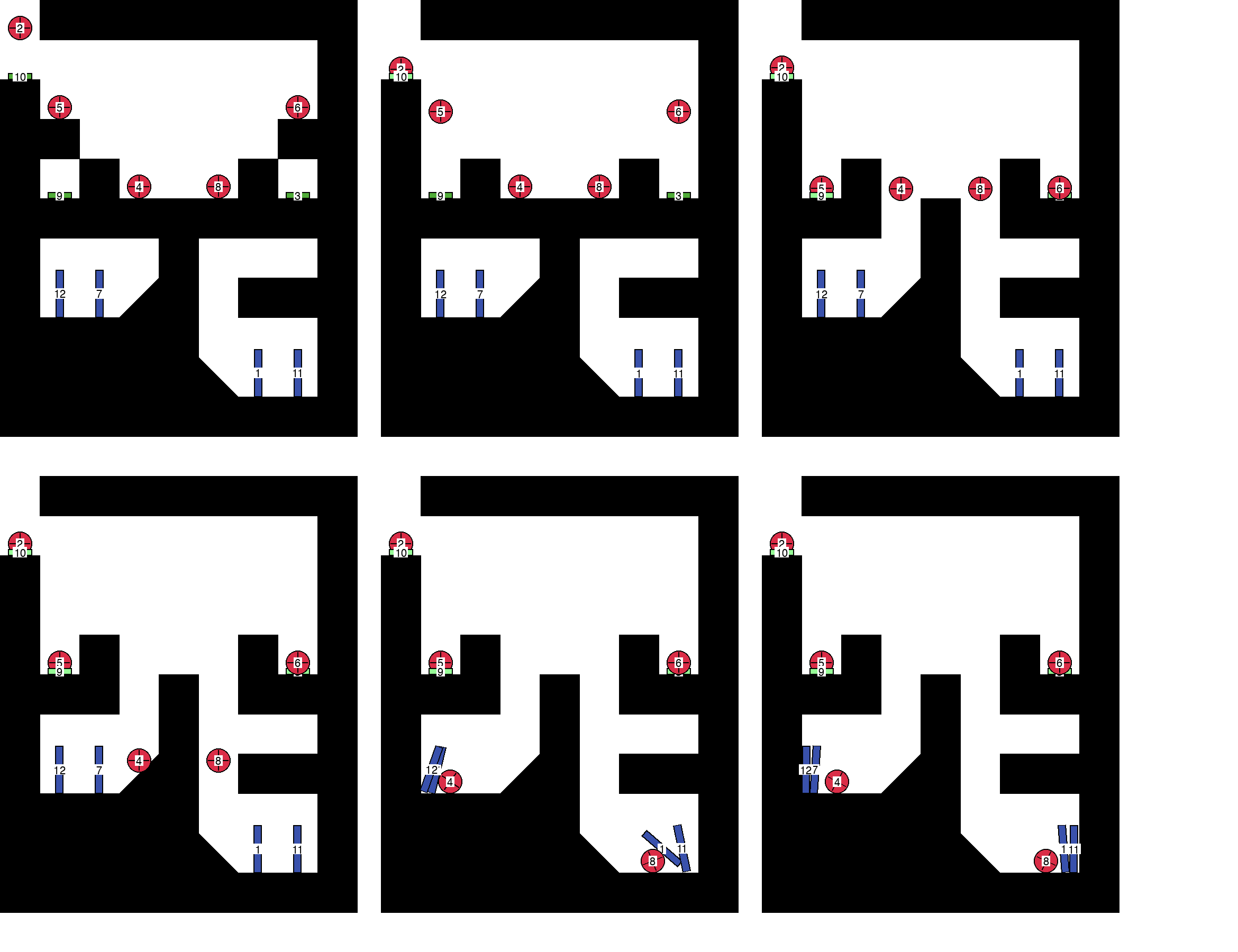}
    \caption{
    \textbf{Observational rollout of Parallel Triggers}
    }\label{fig:env:t2}
\end{figure}

\begin{figure}[H]
    \centering
    \includegraphics[trim=0 0 100 0,clip,width=\linewidth]{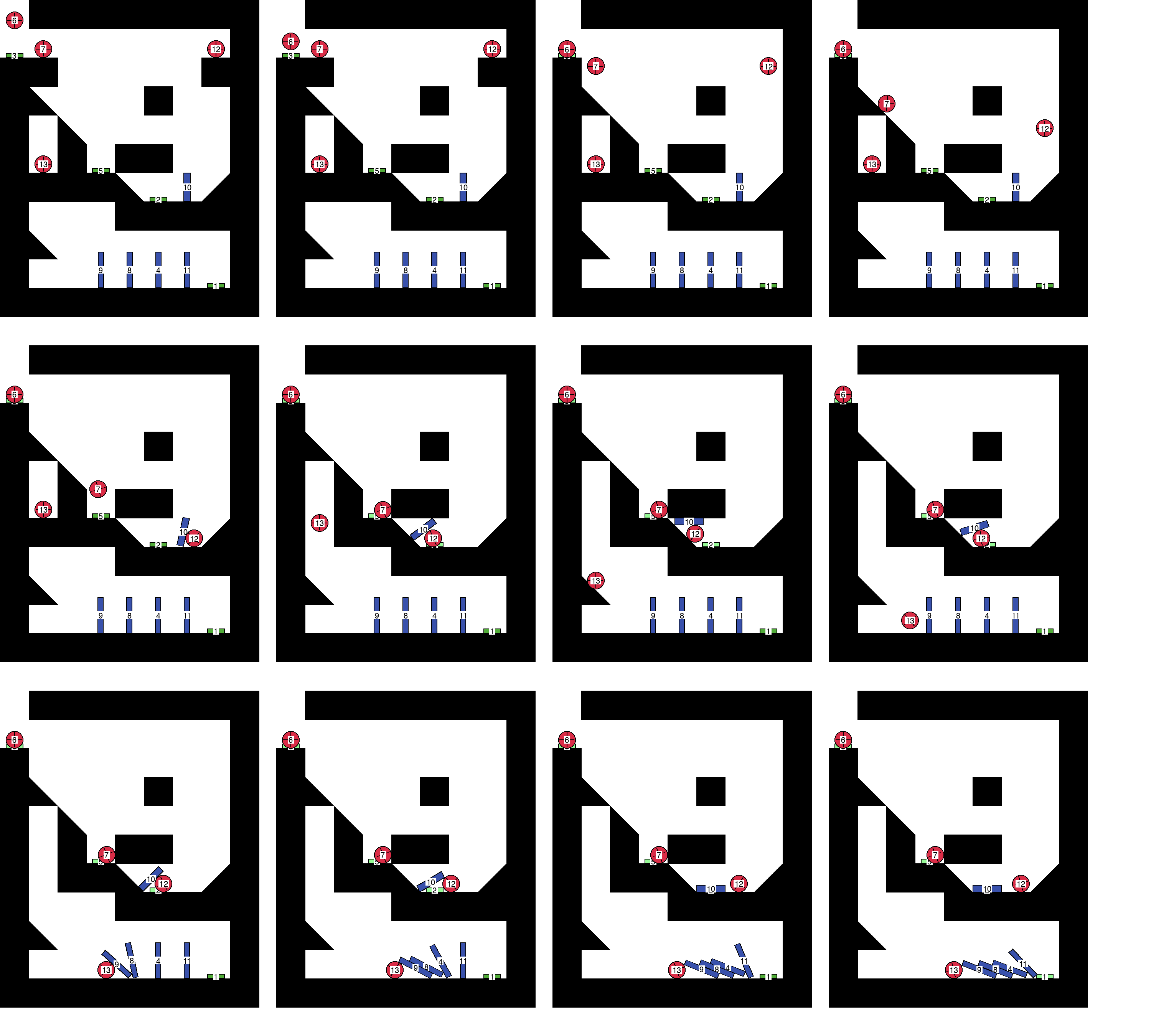}
    \caption{
    \textbf{Observational rollout of Intertwined Mechanisms}
    }\label{fig:env:t3}
\end{figure}

\begin{figure}[H]
    \centering
    \includegraphics[trim=0 0 60 0,clip,width=\linewidth]{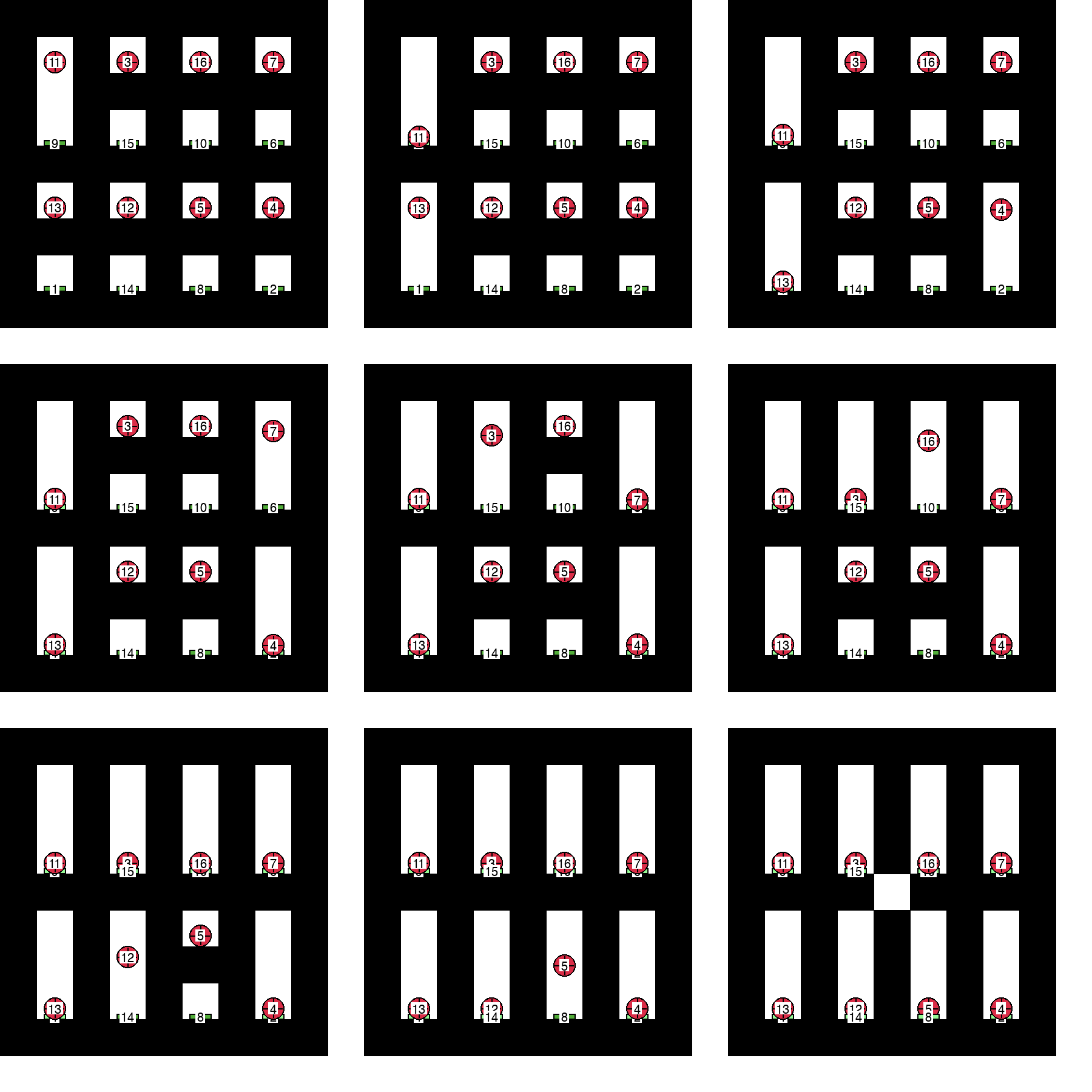}
    \caption{
    \textbf{Observational rollout of Linear Slot-Machine}
    }\label{fig:env:t4}
\end{figure}

\begin{figure}[H]
    \centering
    \includegraphics[trim=0 0 100 0,clip,width=\linewidth]{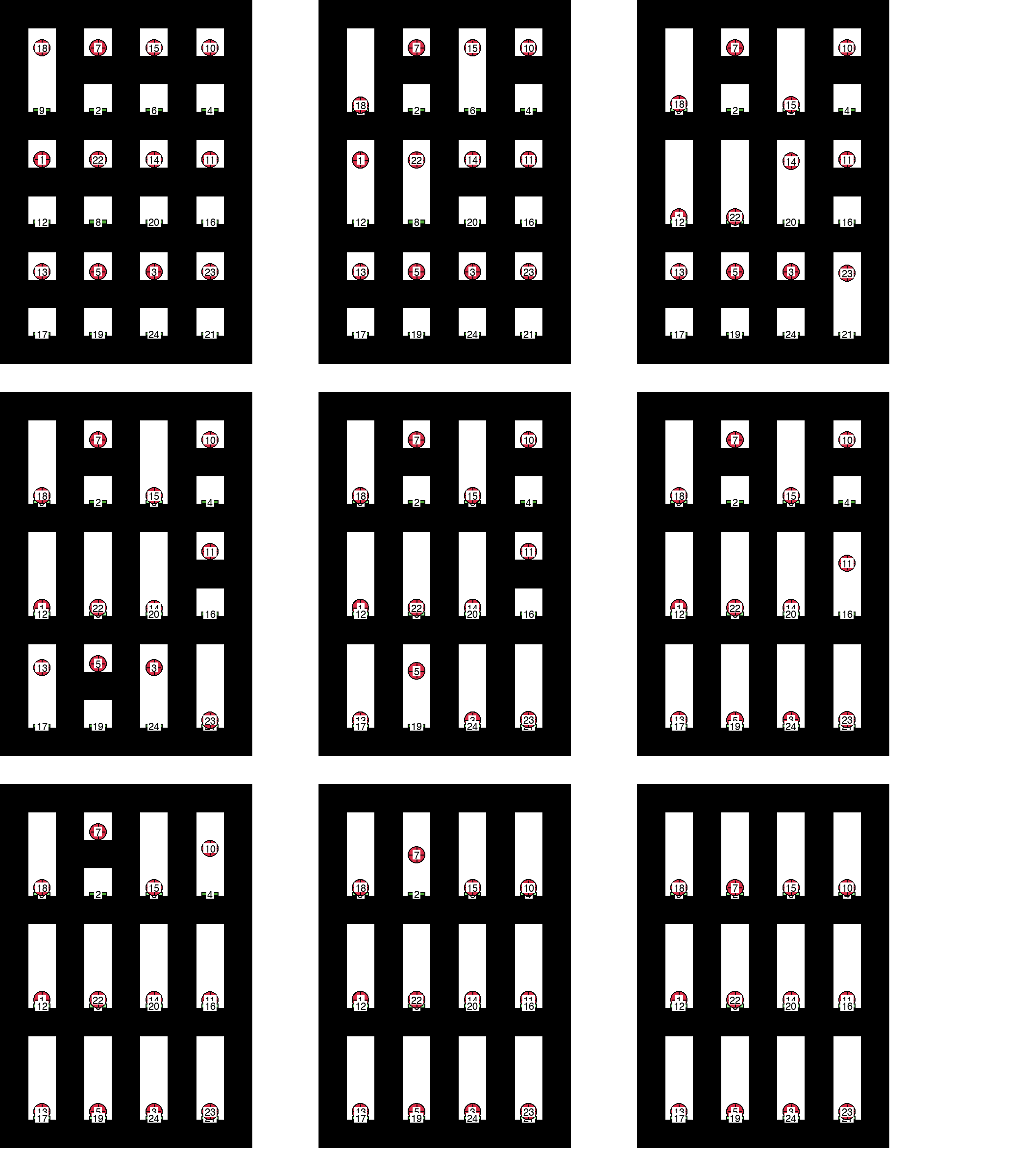}
    \caption{
    \textbf{Observational rollout of Large Slot-Machine}
    }\label{fig:env:t4_big}
\end{figure}

The solutions for the six environments are:
\begin{itemize}
    \item \textbf{Minimal chain (Figure~\ref{fig:env:t1}):}
    $3\rightarrow4\rightarrow1\rightarrow2$.
    
    \item \textbf{Sequential chain (Figure~\ref{fig:env:t0}):} $3\rightarrow9\rightarrow6\rightarrow8\rightarrow4\rightarrow7\rightarrow2\rightarrow1\rightarrow5\rightarrow10\rightarrow11$ (The button $2$ releases ball $1$ by removing the wall underneath it).

    \item \textbf{Parallel triggers (Figure~\ref{fig:env:t2}):} $2\rightarrow10,\;\;10\rightarrow5\rightarrow9\rightarrow8\rightarrow1\rightarrow11,\;\;10\rightarrow6\rightarrow3\rightarrow4\rightarrow7\rightarrow12$

    \item \textbf{Intertwined mechanisms (Figure~\ref{fig:env:t3}):} $6\rightarrow3,\;\;3\rightarrow7\rightarrow5,\;\;3\rightarrow12\rightarrow10\rightarrow2\rightarrow13\rightarrow9\rightarrow8\rightarrow4\rightarrow11\rightarrow1$.

    \item \textbf{Linear Slot-machine (Figure~\ref{fig:env:t4}):} $11\rightarrow9\rightarrow13\rightarrow1\rightarrow4\rightarrow2\rightarrow7\rightarrow6\rightarrow3\rightarrow15\rightarrow16\rightarrow10\rightarrow12\rightarrow14\rightarrow5\rightarrow8$.

    \item \textbf{Large slot-machine (Figure~\ref{fig:env:t4_big}):} $18\rightarrow9,\;\;9\rightarrow1\rightarrow12\rightarrow23\rightarrow21\rightarrow13\rightarrow17,\;\;9\rightarrow22\rightarrow8,\;\;9\rightarrow15\rightarrow6\rightarrow14\rightarrow20\rightarrow3\rightarrow24\rightarrow5\rightarrow19\rightarrow11\rightarrow16\rightarrow10\rightarrow4\rightarrow7\rightarrow2$.
    
\end{itemize}


\section{Proof of Theorem~\ref{thm:sample} (Sample complexity)}\label{app:proofs}

\begin{proof}
    Fix $i\neq j$.
If $j\notin\Desc(i)$, then under the intervention $\doop{X_i=0}$ the activation of $X_j$ is possible with probability
$p_{ij}=\Pr(X_j=1\mid \doop{X_i=0})\ge q_{\min}$.
Across the $n_i$ independent executions under $\doop{X_i=0}$, the outcomes $X_j^{(e)}$ are independent Bernoulli trials with success probability $p_{ij}$.
Therefore,
\[
\Pr(\hat A(i,j)=1)=\Pr(\hat p_{ij}=0)
=\Pr\!\Big(\sum_{e=1}^{n_i} X_j^{(e)}=0\Big)
=(1-p_{ij})^{n_i}\le e^{-p_{ij}n_i}\le e^{-q_{\min}n_i},
\]
which proves~\eqref{eq:expbound}. Now, define the failure event
\[
E \;\coloneqq\; \{\hat A \neq A\}
\;=\;
\bigl\{\exists\, i\neq j : \hat A(i,j) \neq A(i,j)\bigr\}.
\]


If $j \in \Desc(i)$, then whenever the intervention $\doop{X_i=0}$ is performed, then $X_j=0$ (Lemma~\ref{lem:cascade}). Since descendant pairs cannot produce errors, the failure event can only arise from non-descendant pairs. Therefore,

\[
E \;\subseteq\;
\bigcup_{i\neq j:\, j \notin \Desc(i)}
\{\hat A(i,j)=1\}.
\]
Applying the union bound gives
\[
\Pr(\hat A \neq A)
\;\le\;
\sum_{i\neq j:\, j \notin \Desc(i)}
\Pr\bigl(\hat A(i,j)=1\bigr).
\]

For any non-descendant pair $(i,j)$, \eqref{eq:expbound} says that
\(
\Pr(\hat A(i,j)=1) \le \exp(-q_{\min} n_i).
\)
Using $n_i \ge n_{\min}$ for all $i$ and noting that there are at most
$N(N-1)$ ordered pairs with $i\neq j$, we obtain
\[
\Pr(\hat A \neq A)
\;\le\;
N(N-1)\, \exp(-q_{\min} n_{\min}),
\]
which yields~\eqref{eq:probcorrect}.
To ensure $\Pr(\hat A \neq A)\le \delta$, it suffices that
\[
N(N-1)\, \exp(-q_{\min} n_{\min}) \le \delta.
\]
Solving for $n_{\min}$ gives
\[
n_{\min}
\;\ge\;
\frac{1}{q_{\min}}
\Bigl(\log\!\bigl(N(N-1)\bigr) + \log\tfrac{1}{\delta}\Bigr).
\]

\end{proof}

\section{Additional experimental results}\label{app:results}

Section~\ref{sec:experiments} describes the experimental setup and evaluation protocol. In this appendix, we report the complete set of results across all environments and displacement settings we considered in our experiments. PC is omitted from some tables because it failed with a math domain error. SSHD denotes skeleton SHD.

We estimate $\hat q_{\min}$ for each environment and displacement by running 1000 blocking interventions per object, computing $\hat p_{ij}$ for all non-descendant pairs $(i, j)$ using the ground-truth graph, and taking the smallest $\hat p_{ij}$. Note that $\hat q_{\min}$ is estimated from a separate interventional dataset and is reported only to characterize the sample-complexity regime induced by displacement $\Delta$, not as an input to the learning algorithm.

For the synthetic environments considered, $q_{\min}$ can instead be computed exactly from the known causal graph and the uniform node failure probability. Because these graphs are branched, the minimum non-descendant activation probability is achieved by the deepest leaf node, yielding $q_{\min} = (1-p)^{L}$, where $L$ is the length of the longest directed path. For Synthetic Parallel Triggers-0.1 this gives $q_{\min}=0.9^{7}\approx0.478$, while for Synthetic Large Slot-Machine-0.1 it is $q_{\min}=0.9^{16}\approx0.185$.


\begin{table}[h]
\centering
\begin{tabular}{lcccccccc}
\toprule
Method & $\Delta$ & $\hat q_{\min}$ & Precision & Recall & F1 & SHD & SSHD & Time (s) \\
\midrule
\textbf{Our Method} (M=4) & 0.1 & 0.395 & 0.980 & 0.980 & \textbf{0.980} & \textbf{0.09} & \textbf{0.06} & 0.312 \\
Collision-as-influence & 0.1 & 0.395 & 0.970 & 0.647 & 0.776 & 1.06 & 1.00 & 0.000 \\
Temporal-precedence & 0.1 & 0.395 & 1.000 & 0.667 & 0.800 & 1.00 & 1.00 & 0.000 \\
PC & 0.1 & 0.395 & 0.476 & 0.919 & 0.626 & 3.00 & 3.24 & 0.006 \\
\midrule
\textbf{Our Method} (M=3) & 0.2 & 0.396 & 0.973 & 0.973 & \textbf{0.973} & \textbf{0.13} & \textbf{0.10} & 0.324 \\
Collision-as-influence & 0.2 & 0.396 & 0.945 & 0.630 & 0.756 & 1.11 & 1.00 & 0.000 \\
Temporal-precedence & 0.2 & 0.396 & 0.985 & 0.657 & 0.788 & 1.03 & 1.00 & 0.000 \\
PC & 0.2 & 0.396 & 0.459 & 0.863 & 0.598 & 3.00 & 3.41 & 0.005 \\
\midrule
\textbf{Our Method} (M=3) & 0.3 & 0.408 & 0.977 & 0.977 & \textbf{0.977} & \textbf{0.12} & \textbf{0.10} & 0.326 \\
Collision-as-influence & 0.3 & 0.408 & 0.920 & 0.670 & 0.770 & 0.99 & 0.80 & 0.000 \\
Temporal-precedence & 0.3 & 0.408 & 0.963 & 0.693 & 0.801 & 0.92 & 0.83 & 0.000 \\
PC & 0.3 & 0.408 & 0.466 & 0.835 & 0.594 & 3.00 & 3.29 & 0.006 \\
\midrule
\textbf{Our Method} (M=3) & 0.4 & 0.420 & 0.987 & 0.987 & \textbf{0.987} & \textbf{0.06} & \textbf{0.04} & 0.325 \\
Collision-as-influence & 0.4 & 0.420 & 0.912 & 0.737 & 0.807 & 0.79 & 0.56 & 0.000 \\
Temporal-precedence & 0.4 & 0.420 & 0.923 & 0.793 & 0.845 & 0.62 & 0.41 & 0.000 \\
PC & 0.4 & 0.420 & 0.453 & 0.762 & 0.563 & 3.00 & 3.39 & 0.006 \\
\midrule
\textbf{Our Method} (M=2) & 0.5 & 0.424 & 0.963 & 0.963 & \textbf{0.963} & \textbf{0.18} & \textbf{0.14} & 0.290 \\
Collision-as-influence & 0.5 & 0.424 & 0.835 & 0.663 & 0.732 & 1.01 & 0.60 & 0.000 \\
Temporal-precedence & 0.5 & 0.424 & 0.889 & 0.787 & 0.825 & 0.75 & 0.52 & 0.000 \\
PC & 0.5 & 0.424 & 0.430 & 0.729 & 0.535 & 3.00 & 3.58 & 0.006 \\
\bottomrule
\end{tabular}
\caption{Results for Minimal Chain (4 variables, 100 seeds, 16 samples)}
\label{tab:comparison_t1_0_1}
\end{table}

\begin{table}[h]
\centering
\begin{tabular}{lcccccccc}
\toprule
Method & $\Delta$ & $\hat q_{\min}$ & Precision & Recall & F1 & SHD & SSHD & Time (s) \\
\midrule
\textbf{Our Method} (M=1) & 0.1 & 0.457 & 0.996 & 0.995 & \textbf{0.995} & \textbf{0.08} & \textbf{0.07} & 0.325 \\
Collision-as-influence & 0.1 & 0.457 & 0.666 & 0.683 & 0.674 & 4.44 & 2.27 & 0.000 \\
Temporal-precedence & 0.1 & 0.457 & 0.729 & 0.744 & 0.736 & 3.79 & 2.23 & 0.000 \\
\midrule
\textbf{Our Method} (M=1) & 0.2 & 0.457 & 0.994 & 0.995 & \textbf{0.995} & \textbf{0.09} & \textbf{0.07} & 0.329 \\
Collision-as-influence & 0.2 & 0.457 & 0.637 & 0.674 & 0.655 & 4.86 & 2.60 & 0.000 \\
Temporal-precedence & 0.2 & 0.457 & 0.696 & 0.734 & 0.714 & 4.23 & 2.57 & 0.000 \\
\bottomrule
\end{tabular}
\caption{Results for Sequential Chain (11 variables, 100 seeds)}
\label{tab:comparison_t0_0_1}
\end{table}

\begin{table}[h]
\centering
\begin{tabular}{lcccccccc}
\toprule
Method & $\Delta$ & $\hat q_{\min}$ & Precision & Recall & F1 & SHD & SSHD & Time (s) \\
\midrule
\textbf{Our Method} (M=1) & 0.1 & 0.647 & 1.000 & 1.000 & \textbf{1.000} & \textbf{0.00} & \textbf{0.00} & 0.293 \\
Collision-as-influence & 0.1 & 0.647  & 0.870 & 0.588 & 0.701 & 5.00 & 4.47 & 0.000 \\
Temporal-precedence & 0.1 & 0.647  & 0.910 & 0.613 & 0.732 & 4.70 & 4.44 & 0.000 \\
\midrule
\textbf{Our Method} (M=1) & 0.2 & 0.647 & 1.000 & 1.000 & \textbf{1.000} & \textbf{0.00} & \textbf{0.00} & 0.291 \\
Collision-as-influence & 0.2 & 0.647 & 0.870 & 0.588 & 0.701 & 5.00 & 4.47 & 0.000 \\
Temporal-precedence & 0.2 & 0.647 & 0.908 & 0.620 & 0.736 & 4.72 & 4.54 & 0.000 \\
\midrule
\textbf{Our Method} (M=1) & 0.3 & 0.647 & 1.000 & 1.000 & \textbf{1.000} & \textbf{0.00} & \textbf{0.00} & 0.327 \\
Collision-as-influence & 0.3 & 0.647 & 0.853 & 0.594 & 0.699 & 5.16 & 4.69 & 0.000 \\
Temporal-precedence & 0.3 & 0.647 & 0.913 & 0.626 & 0.742 & 4.69 & 4.58 & 0.000 \\
\midrule
\textbf{Our Method} (M=2) & 0.4 & 0.622 & 0.999 & 0.999 & \textbf{0.999} & \textbf{0.02} & \textbf{0.02} & 0.295 \\
Collision-as-influence & 0.4 & 0.622 & 0.807 & 0.594 & 0.684 & 5.59 & 5.12 & 0.000 \\
Temporal-precedence & 0.4 & 0.622 & 0.808 & 0.620 & 0.699 & 5.73 & 5.55 & 0.000 \\
\bottomrule
\end{tabular}
\caption{Results for Parallel Triggers (12 variables, 100 seeds)}
\label{tab:comparison_t2_0_1}
\end{table}

\begin{table}[h]
\centering
\begin{tabular}{lcccccccc}
\toprule
Method & $\Delta$ & $\hat q_{\min}$ & Precision & Recall & F1 & SHD & SSHD & Time (s) \\
\midrule
\textbf{Our Method} (M=2) & 0.1 & 0.590 & 0.999 & 0.999 & \textbf{0.999} & \textbf{0.02} & \textbf{0.02} & 0.299 \\
Collision-as-influence & 0.1 & 0.590 & 0.694 & 0.634 & 0.663 & 6.35 & 4.96 & 0.000 \\
Temporal-precedence & 0.1 & 0.590 & 0.722 & 0.661 & 0.690 & 6.05 & 4.98 & 0.000 \\
PC & 0.1 & 0.590 & 0.154 & 1.000 & 0.267 & 12.00 & 66.00 & 7.099 \\
\midrule
\textbf{Our Method} (M=2) & 0.2 & 0.588 & 0.999 & 0.999 & \textbf{0.999} & \textbf{0.02} & \textbf{0.02} & 0.296 \\
Collision-as-influence & 0.2 & 0.588 & 0.691 & 0.633 & 0.661 & 6.40 & 5.00 & 0.000 \\
Temporal-precedence & 0.2 & 0.588 & 0.736 & 0.671 & 0.702 & 5.90 & 4.95 & 0.000 \\
PC & 0.2 & 0.588 & 0.153 & 0.994 & 0.265 & 12.00 & 66.07 & 7.158 \\
\bottomrule
\end{tabular}
\caption{Results for Intertwined Mechanisms (13 variables, 100 seeds)}
\label{tab:comparison_t3_0_1}
\end{table}

\begin{table}[h]
\centering
\begin{tabular}{lcccccccc}
\toprule
Method & $\Delta$ & $\hat q_{\min}$ & Precision & Recall & F1 & SHD & SSHD & Time (s) \\
\midrule
\textbf{Our Method} (M=1) & 0.1 & 0.465 & 1.000 & 1.000 & \textbf{1.000} & \textbf{0.00} & \textbf{0.00} & 0.337 \\
Collision-as-influence & 0.1 & 0.465 & 1.000 & 0.533 & 0.696 & 7.00 & 7.00 & 0.000 \\
Temporal-precedence & 0.1 & 0.465 & 1.000 & 0.533 & 0.696 & 7.00 & 7.00 & 0.000 \\
\midrule
\textbf{Our Method} (M=1) & 0.2 & 0.465 & 1.000 & 1.000 & \textbf{1.000} & \textbf{0.00} & \textbf{0.00} & 0.300 \\
Collision-as-influence & 0.2 & 0.465 & 1.000 & 0.533 & 0.696 & 7.00 & 7.00 & 0.000 \\
Temporal-precedence & 0.2 & 0.465 & 1.000 & 0.533 & 0.696 & 7.00 & 7.00 & 0.000 \\
\midrule
\textbf{Our Method} (M=1) & 0.3 & 0.465 & 1.000 & 1.000 & \textbf{1.000} & \textbf{0.00} & \textbf{0.00} & 0.303 \\
Collision-as-influence & 0.3 & 0.465& 1.000 & 0.533 & 0.696 & 7.00 & 7.00 & 0.000 \\
Temporal-precedence & 0.3 & 0.465 & 1.000 & 0.533 & 0.696 & 7.00 & 7.00 & 0.000 \\
\midrule
\textbf{Our Method} (M=2) & 0.4 & 0.458 & 0.999 & 0.999 & \textbf{0.999} & \textbf{0.03} & \textbf{0.03} & 0.297 \\
Collision-as-influence & 0.4 & 0.458 & 1.000 & 0.533 & 0.696 & 7.00 & 7.00 & 0.000 \\
Temporal-precedence & 0.4 & 0.458 & 0.888 & 0.622 & 0.726 & 7.00 & 7.00 & 0.000 \\
\bottomrule
\end{tabular}
\caption{Results for Linear Slot-Machine (16 variables, 100 seeds)}
\label{tab:comparison_t4_0_1}
\end{table}

\begin{table}[h]
\centering
\begin{tabular}{lcccccccc}
\toprule
Method & $\Delta$ & $\hat q_{\min}$ & Precision & Recall & F1 & SHD & SSHD & Time (s) \\
\midrule
\textbf{Our Method} (M=1) & 0.1 & 0.686 & 1.000 & 1.000 & \textbf{1.000} & \textbf{0.00} & \textbf{0.00} & 0.316 \\
Collision-as-influence & 0.1 & 0.686 & 1.000 & 0.522 & 0.686 & 11.00 & 11.00 & 0.000 \\
Temporal-precedence & 0.1 & 0.686 & 1.000 & 0.522 & 0.686 & 11.00 & 11.00 & 0.000 \\
\midrule
\textbf{Our Method} (M=1) & 0.2 & 0.686 & 1.000 & 1.000 & \textbf{1.000} & \textbf{0.00} & \textbf{0.00} & 0.297 \\
Collision-as-influence & 0.2 & 0.686 & 1.000 & 0.522 & 0.686 & 11.00 & 11.00 & 0.000 \\
Temporal-precedence & 0.2 & 0.686 & 1.000 & 0.522 & 0.686 & 11.00 & 11.00 & 0.000 \\
\midrule
\textbf{Our Method} (M=1) & 0.3 & 0.686 & 1.000 & 1.000 & \textbf{1.000} & \textbf{0.00} & \textbf{0.00} & 0.297 \\
Collision-as-influence & 0.3 & 0.686 & 1.000 & 0.522 & 0.686 & 11.00 & 11.00 & 0.000 \\
Temporal-precedence & 0.3 & 0.686 & 1.000 & 0.522 & 0.686 & 11.00 & 11.00 & 0.000 \\
\midrule
\textbf{Our Method} (M=2) & 0.4 & 0.673 & 0.998 & 0.998 & \textbf{0.998} & \textbf{0.09} & \textbf{0.09} & 0.296 \\
Collision-as-influence & 0.4 & 0.673 & 1.000 & 0.522 & 0.686 & 11.00 & 11.00 & 0.000 \\
Temporal-precedence & 0.4 & 0.673 & 0.724 & 0.564 & 0.629 & 15.46 & 15.46 & 0.000 \\
\bottomrule
\end{tabular}
\caption{Results for Large Slot-Machine (24 variables, 100 seeds)}
\label{tab:comparison_t4_big_0_1}
\end{table}






\clearpage

\section{Additional Dataset Example: Parallel Triggers}
\label{app:dataset}

\begin{figure}[H]
\centering
\includegraphics[trim=0 50 480 0,clip,width=0.45\linewidth]{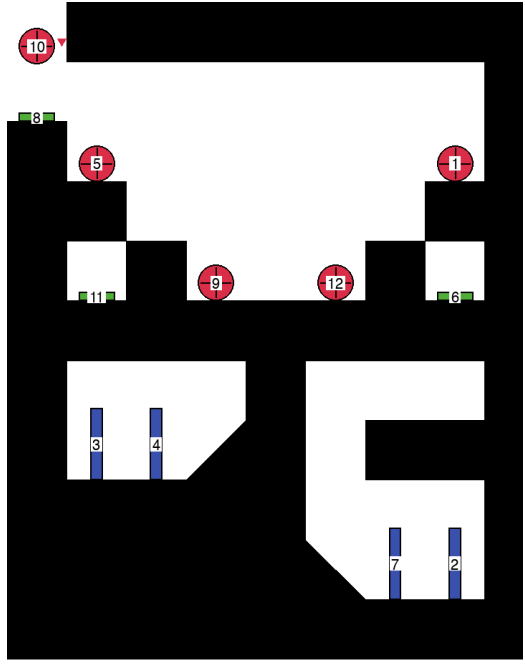}
\caption{
\textbf{Parallel Triggers.}
}
\label{fig:parallel_triggers_app}
\end{figure}

We provide a more detailed example of an interventional dataset (as defined in Section~\ref{sec:dataset}) for a chain-reaction system with \emph{parallel triggers}, illustrated in Figure~\ref{fig:parallel_triggers_app}. In this system, button 8 triggers two branches in parallel:
ball~5 presses button~11, releasing ball~12, while ball~1 presses button~6, releasing ball~9.
Buttons~11 and~6 are typically pressed at approximately the same time.
The causal structure of this system is
\[
10 \rightarrow 8,\quad
8 \rightarrow 5 \rightarrow 11 \rightarrow 12 \rightarrow 7 \rightarrow 2,\quad
8 \rightarrow 1 \rightarrow 6 \rightarrow 9 \rightarrow 4 \rightarrow 3.
\]

\paragraph{Dataset structure.}
Each execution $e$ yields a pair $(I_e, X^{(e)})$,
where $I_e \in \{1,\dots,N\} \cup \{\varnothing\}$ denotes the intervened object
and $X^{(e)} \in \{0,1\}^N$ records which objects became active during the execution.

\paragraph{Example executions.}
A typical observational execution may produce an activation vector of the form
\[
(\varnothing,\; X^{(e)}) \quad\text{with}\quad
X^{(e)}_{11}=X^{(e)}_{6}=1,
\]
followed by activation of their respective downstream cascades (e.g., both balls 9 and 12 move).
Under a blocking intervention such as $\doop{X_{11}=0}$, the dataset instead contains samples of the form
\[
(11,\; X^{(e)}) \quad\text{with}\quad
X^{(e)}_{12}=X^{(e)}_{7}=X^{(e)}_{2}=0,
\]
while the parallel branch through button~6 will be active ($X^{(e)}_{9}=X^{(e)}_{4}=X^{(e)}_{3}=1$).
Analogously, $\doop{X_1=0}$ suppresses the branch releasing ball~9 while leaving the other branch intact.
This example illustrates how interventional datasets in chain-reaction systems capture
selective suppression of downstream activations under blocking interventions,
even in environments with simultaneous or parallel triggering events.

\paragraph{Example dataset.}
An example interventional dataset for this environment may contain samples such as
\[
\begin{aligned}
(\varnothing,\; X) &= (1,1,1,1,1,1,1,1,1,1,1,1), \\
(\varnothing,\; X) &= (1,0,1,1,1,1,1,1,1,1,1,1), \\
(11,\; X) &= (1,0,1,1,1,1,0,1,1,1,0,0), \\
(6,\; X) &= (1,1,0,0,1,0,1,1,0,1,1,1), \\
(10,\; X) &= (0,0,0,0,0,0,0,0,0,0,0,0),
\end{aligned}
\]
where each vector is ordered according to object indices $1$ through $12$. 
Note that in the second observational sample, the absence of activation of domino~2 may arise from stochastic failure (e.g., if domino~2 is placed too far during resetting, domino~7 may fail to reach it).

\end{document}